\documentclass[letterpaper]{article} 
\usepackage[arxiv]{aaai23}  
\usepackage{times}  
\usepackage{helvet}  
\usepackage{courier}  
\usepackage[hyphens]{url}  
\usepackage{graphicx} 
\usepackage{natbib}  
\usepackage{caption} 
\usepackage{algorithm}

\usepackage{newfloat}
\usepackage{listings}
\DeclareCaptionStyle{ruled}{labelfont=normalfont,labelsep=colon,strut=off} 
\floatstyle{ruled}
\newfloat{listing}{tb}{lst}{}
\floatname{listing}{Listing}
\usepackage{amsfonts}
\usepackage{amsmath}
\usepackage{bm}
\usepackage{algpseudocode}
\newcommand{\argmin}{\mathop{\rm arg~min}\limits}
\newtheorem{definition}{Definition}
\newtheorem{lemma}{Lemma}
\newtheorem{theorem}{Theorem}
\newtheorem{corollary}{Corollary}
\newtheorem{proof}{Proof}

\def\qed{\hfill $\Box$}
 \usepackage{booktabs}
 \usepackage{arydshln}
 \usepackage{multirow}
\usepackage[tight,normalsize]{subfigure}
  
\title{Fast Regularized Discrete Optimal Transport with Group-Sparse Regularizers}
\author{
    Yasutoshi Ida\textsuperscript{\rm 1},
    Sekitoshi Kanai\textsuperscript{\rm 1},
    Kazuki Adachi\textsuperscript{\rm 1},
    Atsutoshi Kumagai\textsuperscript{\rm 1},
    Yasuhiro Fujiwara\textsuperscript{\rm 2}
}
\affiliations{
     \textsuperscript{\rm 1}NTT Computer and Data Science Laboratories\\
     \textsuperscript{\rm 2}NTT Communication Science Laboratories\\

    yasutoshi.ida@ieee.org, \{sekitoshi.kanai.fu, kazuki.adachi.xy, atsutoshi.kumagai.ht, yasuhiro.fujiwara.kh\}@hco.ntt.co.jp
}

\usepackage{bibentry}

\begin{document}

\maketitle

\begin{abstract}
Regularized discrete optimal transport (OT) is a powerful tool to measure the distance between two discrete distributions that have been constructed from data samples on two different domains. While it has a wide range of applications in machine learning, in some cases the sampled data from only one of the domains will have class labels such as unsupervised domain adaptation. In this kind of problem setting, a group-sparse regularizer is frequently leveraged as a regularization term to handle class labels. In particular, it can preserve the label structure on the data samples by corresponding the data samples with the same class label to one group-sparse regularization term. As a result, we can measure the distance while utilizing label information by solving the regularized optimization problem with gradient-based algorithms. However, the gradient computation is expensive when the number of classes or data samples is large because the number of regularization terms and their respective sizes also turn out to be large. This paper proposes fast discrete OT with group-sparse regularizers. Our method is based on two ideas. The first is to safely skip the computations of the gradients that must be zero. The second is to efficiently extract the gradients that are expected to be nonzero. Our method is guaranteed to return the same value of the objective function as that of the original method. Experiments show that our method is up to 8.6 times faster than the original method without degrading accuracy.
\end{abstract}

\section{Introduction}
\label{introduction}
Regularized discrete optimal transport (OT) is a powerful tool to compute the distance between two discrete probability distributions that have been constructed from data samples on two different domains.
It seeks a map, called a transportation plan, for moving the probability mass of one distribution to that of another distribution with the cheapest cost.
Once the transportation plan is obtained, the data samples on one domain can be transported to those on another domain,
and their distance can be computed as the cost based on the transportation plan.
Owing to its theoretical foundations and desirable properties, 
it has received much attention in various fields, including shape recognition \cite{gangbo2006shape}, color transfer \cite{pitie2007autimated}, domain adaptation \cite{courty2017optimal}, and human activity recognition \cite{wang2021cross}.
\renewcommand{\thefootnote}{\fnsymbol{footnote}}
\footnote[0]{This is an extended version of the paper accepted by the 37th AAAI Conference on Artificial Intelligence (AAAI 2023).}
\renewcommand{\thefootnote}{\arabic{footnote}}

When we construct two distributions from data samples of two different domains on which we want to use regularized discrete OT, in some cases the sampled data from only one of the domains will have class labels, e.g., samples in the target domain are unlabeled, but labeled samples in similar domains can be observed.
Such situations are frequently encountered in unsupervised/semi-supervised domain adaptation, which has received much attention in the machine-learning community \cite{courty2017optimal}.
Although it is expected that utilizing not only data samples but also their class labels to obtain the transportation plan will improve the performance of such applications, the plain regularized discrete OTs, such as entropy regularized OT \cite{cuturi2013sinkhorn}, cannot handle class labels.
To incorporate the information of the class labels into the transportation plan, a group-sparse regularizer is frequently used as a regularization term \cite{courty2014domain,courty2017optimal,redko2017theoretical,blondel2018smooth,das2018sample,das2018unsupervised,alaya2019screening,li2020optimal,wang2021cross}.
In this approach, one group-sparse regularization term corresponds to one class label and the regularizer induces group-sparsity in the transportation plan.
Specifically, this regularizer considers data samples with the same label as one group and restricts the transportation plan so that the samples tend to be transported to the same data sample on another domain (Figure~\ref{fig:toy}).
Since this approach preserves the structure of the label information on the data samples during the transportation, the group-sparse regularizer can effectively handle the problem of discrete OT with class labels.

While the discrete OT with a group-sparse regularizer is crucial for handling class labels, it incurs a high computation cost when solving the regularized optimization problem with a gradient-based algorithm, especially when there are large numbers of class labels \cite{olga2015imagnet} or data samples in each class \cite{luca2017scaling}.
This is because the gradient-based algorithm must compute gradient vectors corresponding to the regularization terms.
Namely, when the numbers of class labels and number of data samples in each class increase, the number of gradient vectors and their sizes also increase.
In addition, since the existing algorithms, such as L-BFGS \cite{blondel2018smooth}, need to iteratively compute these gradient vectors to update parameters until convergence,
the computation cost of the gradient computation becomes dominant in the total computation cost as the number of class labels or data samples in each class increases.
\begin{figure}[!t]
  \centering
  \subfigure{\includegraphics[width=0.49\linewidth]{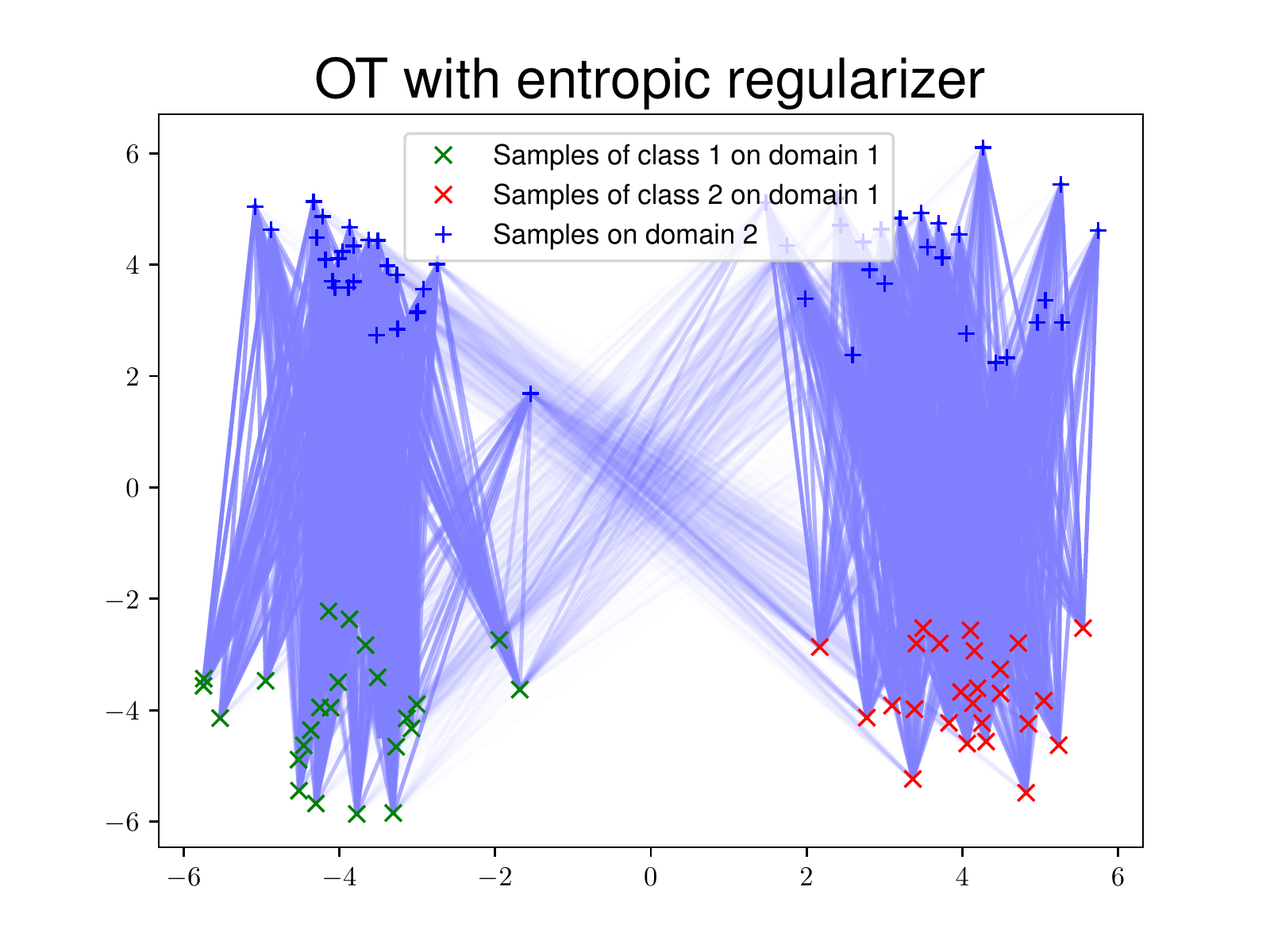}}
  \subfigure{\includegraphics[width=0.49\linewidth]{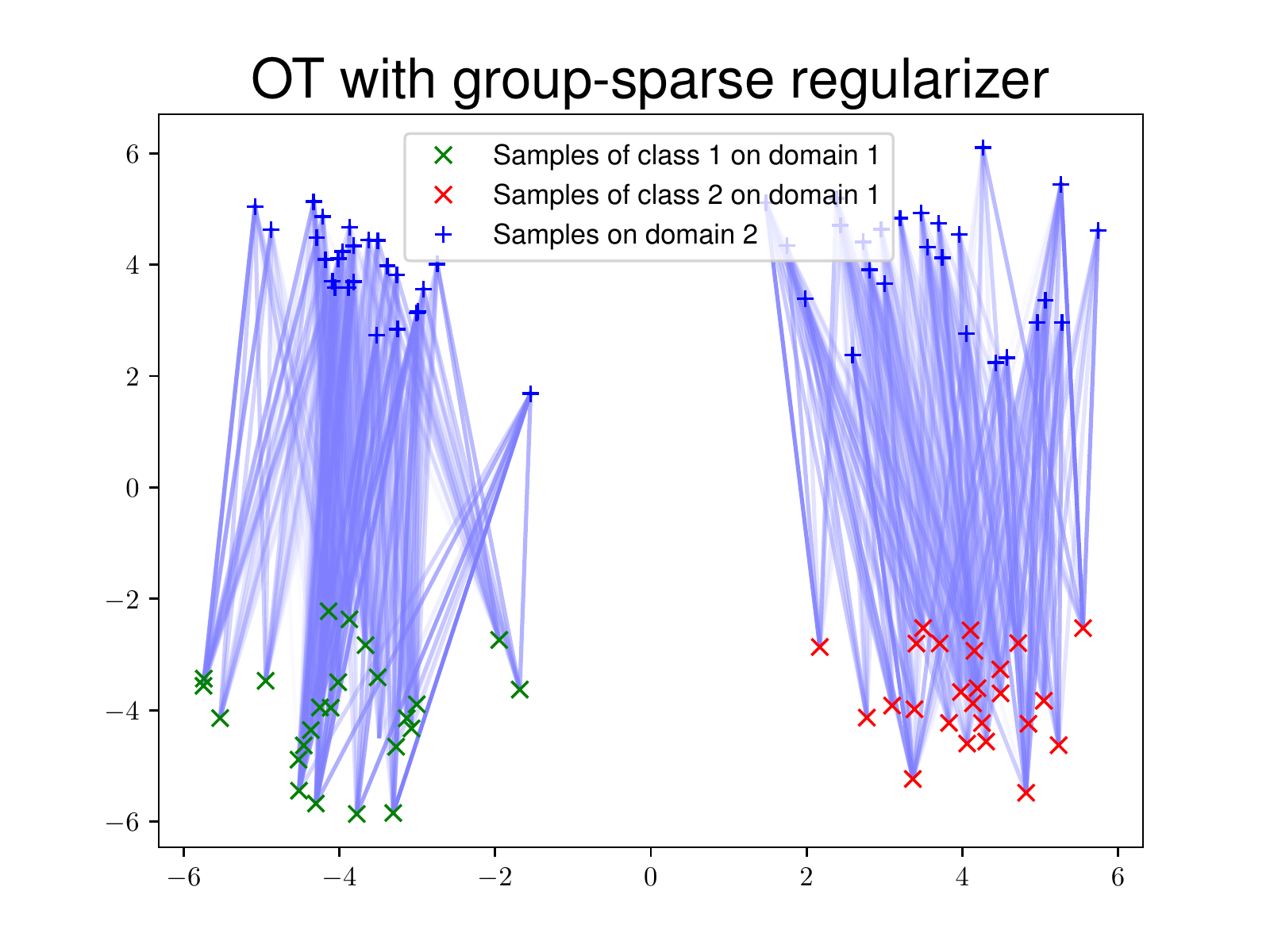}}
  \caption{Map of transportation between samples on domain 1 with two classes and those with two clusters on domain 2 for OT with entropic (left) and group-sparse (right) regularizers. In the left figure, a sample on domain 2 is transported from samples with different class labels, whereas it is transported from the same class in the right figure. These figures suggest that the group-sparse regularizer effectively preserves label information via the structured sparse transportation plan.}
  \label{fig:toy}
\end{figure}

This paper proposes a fast algorithm for discrete OT with a group-sparse regularizer.
To reduce the processing time, we utilize the observation that a large number of gradient vectors of the regularization terms turn out to be zero vectors during optimization.
This is because the gradient vectors are computed with a soft-thresholding function such as the one used in Group Lasso \cite{yuan2006model,ida2019fast,ida2020fast}, which induces sparsity in the gradient vectors.
Based on the observation, our method leverages two ideas.
The first idea is to safely skip gradient computations for groups whose gradient vectors must be zero vectors.
A gradient vector for a group can be quickly checked to be a zero vector or not by approximately computing the soft-thresholding function.
If it is determined to be a zero vector, our method skips the gradient computation; otherwise, it exactly computes the gradient vector.
Here, when computing the exact gradient vector for the latter case, the checking procedure is an extra overhead.
To alleviate this overhead, we introduce the second idea; we construct a subset of groups whose gradient vectors turn out to be nonzero vectors during optimization.
We can reduce the overhead of the first idea by computing gradient vectors without the checking procedure for the specified set.
Theoretically, our method is guaranteed to return the same value of the objective function as the original method.
Experiments show that our method is up to 8.6 times faster than the original method without degrading accuracy.

\textbf{Notation.}
We denote scalars, vectors, and matrices with lower-case, bold lower-case and upper-case letters, \textit{e.g.}, \(t\), \(\bm{t}\) and \(T\), respectively.  
Given a matrix \(T\), we denote its elements by \(t_{i, j}\) and its columns by \(\bm{t}_{j}\).
Given a vector \(\bm{t}\), we denote its elements in the \(l\)-th group by \(\bm{t}_{[l]}\) when \(\bm{t}\) decomposes over groups \(l\in \mathcal{L}\) where \(\mathcal{L}\) is the set of group indices.
\(\bm{1}_{m}\) represents an \(m\)-dimensional vector whose elements are ones.
We define \([\bm{x}]_+ := \max(\bm{x}, \bm{0})\) and  \([\bm{x}]_{-} := \min(\bm{x}, \bm{0})\), performed element-wise.

\section{Preliminary}
First, we explain the transportation problem between two discrete probability distributions with label information.
Next, we describe discrete OT with group-sparse regularizers that can handle label information.
Finally, we introduce the smooth relaxed dual formulation of the discrete OT problem that is the key formulation to our method.

\subsection{Transportation Problem with Label Information}
We introduce two sets of data samples \(\bm{X}_{S}=\{\bm{x}^{(i)}_{S}\in\mathbb{R}^{d}\}_{i=1}^{m}\) associated with class labels \(\bm{Y}_{S}=\{y^{(i)}_{S}\in \mathcal{L}\}_{i=1}^{m}\) and unlabeled samples \(\bm{X}_{T}=\{\bm{x}^{(i)}_{T}\in\mathbb{R}^{d}\}_{i=1}^{n}\).
Note that \(\mathcal{L}=\{1,...,|\mathcal{L}|\}\) is the set of class labels.

Let us consider the transportation problem from \(\bm{X}_{S}\) with \(\bm{Y}_{S}\) to \(\bm{X}_{T}\) by following the previous work \cite{courty2017optimal}.
There, one first constructs two discrete probability distributions from the data samples
\(\textstyle{\bm{a} = \frac{1}{m}\sum_{i=1}^{m} \delta_{\bm{x}_{S}^{(i)}}}\) and \(\textstyle{\bm{b} = \frac{1}{n}\sum_{i=1}^{n} \delta_{\bm{x}_{T}^{(i)}}}\)
where \(\delta_{\bm{x}^{(i)}}\) is the Dirac delta function at location \(\bm{x}^{(i)}\).
These distributions have masses of \(1/m\) and \(1/n\) for each \(\delta_{\bm{x}_{S}^{(i)}}\) and \(\delta_{\bm{x}_{T}^{(i)}}\), respectively.

Discrete OT seeks a map of the minimal cost in moving the mass of \(\bm{a}\) to that of \(\bm{b}\).
As a result, we can transport \(\bm{X}_{S}\) to \(\bm{X}_{T}\) by using the map and compute the distance from the map as the cost.
In Kantorovich's formulation, discrete OT is cast as a linear program (LP), as follows:
\begin{eqnarray}
\label{eq:lp}
\textstyle{\min_{T \in \mathcal{U}(\bm{a},\bm{b})} \langle T, C \rangle_{F}}
\end{eqnarray}
where \(\langle \cdot, \cdot \rangle_{F}\) is the Frobenius dot product and \(\mathcal{U}(\bm{a},\bm{b}) := \left\{T \in \mathbb{R}^{m \times n}_+ \colon T \bm{1}_n = \bm{a}, T^{\top} \bm{1}_m = \bm{b} \right\}\) is the transportation polytope.
\(T\) is a map called the transportation plan and \(C \in \mathbb{R}_+^{m \times n}\) is a cost matrix.
We set \(c_{i, j}:=|\!|\bm{x}^{(i)}_{S}-\bm{x}^{(j)}_{T}|\!|_{2}^{2}\) as in the previous work \cite{courty2017optimal}.
The objective value after solving Problem~(\ref{eq:lp}) is the distance between the two discrete probability distributions.
From here, we treat the samples as matrices, like \(X^{S}\in\mathbb{R}^{m \times d}\), \(\bm{y}^{S}\in\mathbb{R}^{m}\) and \(X^{T}\in\mathbb{R}^{n \times d}\) that correspond to \(\bm{X}_{S}\), \(\bm{Y}_{S}\) and \(\bm{X}_{T}\), respectively.
In this case, when we obtain a transportation plan \(\hat{T}\) by solving the LP problem, \(X^{S}\) is transported to \(X^{T}\) as \(n\hat{T}^{\top}X^{S}\).
However, this formulation cannot handle class labels \(\bm{y}^{S}\), which may improve accuracy of some applications, such as unsupervised domain adaptation.
The next section describes how to find a transportation plan via discrete OT while considering class labels.

\subsection{Discrete OT with Group-sparse Regularizer}
To handle class labels \(\bm{y}^{S}\), a group-sparse regularizer \(\Psi(\cdot)\) \cite{blondel2018smooth} can be incorporated in Problem~(\ref{eq:lp}) as follows:
\begin{eqnarray}
\label{eq:primal_gl}
\min_{T \in \mathcal{U}(\bm{a},\bm{b})} \langle T, C \rangle_{F}+ \textstyle{\sum_{j=1}^n} \Psi(\bm{t}_{j}),
\end{eqnarray}
where
\begin{eqnarray}
\label{eq:reg_gl}
\textstyle{\Psi(\bm{t}_{j})=\gamma(\frac{1}{2} |\!|\bm{t}_{j}|\!|_{2}^{2} + \mu \sum_{l\in\mathcal{L}}|\!|\bm{t}_{j[l]}|\!|_{2}).}
\end{eqnarray}
\(\gamma > 0\) and \(\mu > 0\) are hyperparameters for the regularization terms.
\(\bm{t}_{j[l]}\) is a vector containing elements corresponding to the labeled data associated with the class label \(l\in\mathcal{L}\).
The term \(|\!|\bm{t}_{j[l]}|\!|_{2}\) can be expected to make \(\bm{t}_{j[l]}\) a zero vector, as in Group Lasso~\cite{yuan2006model}.
Therefore, we can expect the transportation plan \(T\) to be a sparse one reflecting the group structure, which corresponds to the structure of data samples with class labels.
Specifically, this regularizer considers data samples \(X^{S}\) with the same label as one group and restricts the transportation plan \(T\) so that they tend to be transported to the same data sample of \(X^{T}\) (Figure~\ref{fig:toy}).
Although solving Problem~(\ref{eq:primal_gl}) does not seem easy, the next section shows that it can easily be solved by considering the smooth relaxed dual of Problem~(\ref{eq:primal_gl}).

\subsection{Smooth Relaxed Dual Formulation}
The smooth relaxed dual of Problem~(\ref{eq:primal_gl}) is as follows \cite{blondel2018smooth}:
\begin{eqnarray}
\label{eq:smoothed_dual}
\max_{\bm{\alpha} \in \mathbb{R}^m, \bm{\beta} \in \mathbb{R}^n} \bm{\alpha}^{\top} \bm{a} + \bm{\beta}^{\top} \bm{b} - \textstyle{\sum_{j=1}^{n}} \psi(\bm{\alpha} + \beta_{j} \bm{1}_{m} - \bm{c}_{j}).
\end{eqnarray}
In the above problem, \(\psi(\bm{f}):=\sup_{\bm{g}\geq 0} \bm{f}^{\top}\bm{g}-\Psi(\bm{g})\) is the convex conjugate of \(\Psi(\cdot)\).
Specifically, it is computed as \(\psi(\bm{f})=\bm{f}^{\top}\bm{g}^{\star}-\Psi(\bm{g}^{\star})\),
where \(\bm{g}^{\star}\) decomposes over groups \(l\in\mathcal{L}\) and equals:
\begin{eqnarray}
\label{eq:gradient}
\bm{g}^\star_{[l]} &=& \argmin_{\bm{g}_{[l]}} \textstyle{\frac{1}{2}} |\!|\bm{g}_{[l]} - \bm{f}_{[l]}^{+}|\!|_{2}^{2} + \mu |\!|\bm{g}_{[l]}|\!|_{2} \nonumber \\
 &=& [1 - \mu/|\!|\bm{f}_{[l]}^{+}|\!|_{2}]_{+} \bm{f}_{[l]}^{+}=\nabla\psi(\bm{f})_{[l]}.
\end{eqnarray}
In the above equations, \(\textstyle{\bm{f}^{+} = \frac{1}{\gamma} [\bm{f}]_{+}}\) and
the third equation is derived from the definition of \(\psi(\cdot)\) and Danskin's theorem \cite{bertsekas1999nonlinear}.
The optimal transportation plan \(T^{\star}\) of Problem~(\ref{eq:primal_gl}) can be recovered from the optimal solutions \(\bm{\alpha}^{\star}\) and \(\bm{\beta}^{\star}\) by computing
\(
\bm{t}_{j}^{\star} = 
\nabla \psi(\bm{\alpha}^{\star} + \beta_{j}^{\star} \bm{1}_{m} - \bm{c}_{j})
\)
for all \(j\in \{1,...,n\}\).
Problem~(\ref{eq:smoothed_dual}) is a differentiable and concave optimization problem without hard constraints.
In addition, we can compute the gradient \(\nabla \psi(\bm{\alpha} + \beta_{j} \bm{1}_{m} - \bm{c}_{j})\) in a closed-form expression by using Equation~(\ref{eq:gradient}).
Therefore, we can use gradient-based algorithms, such as L-BFGS \cite{liu1989on}, to solve Problem~(\ref{eq:smoothed_dual}).

Equation~(\ref{eq:gradient}) can be regarded as a soft-thresholding function \cite{fujiwara2016fast,fujiwara2016fast2}.
Since the \([1-\mu/|\!|\bm{f}_{[l]}^{+}|\!|_{2}]_{+}\) part in Equation~(\ref{eq:gradient}) is represented as \(\max(1-\mu/|\!|\bm{f}_{[l]}^{+}|\!|_{2}, 0)\), a lot of gradient vectors \(\nabla \psi(\bm{\alpha}+\beta_{j} \bm{1}_{m}\!- \bm{c}_{j})_{[l]}\) turn out to be zero vectors during optimization.
However, the gradient vector \(\nabla \psi(\bm{\alpha}+\beta_{j} \bm{1}_{m}\!-\bm{c}_{j})_{[l]}\) is computed for all \(l\in\mathcal{L}\) in every iteration until convergence.
If \(g\) is the number of data samples per class, \(\mathcal{O}(|\mathcal{L}|ng)\) time is required for computing all the gradient vectors in one iteration.
When the total computation cost other than the gradient computation of the solver is \(\mathcal{O}(s_{\rm{s}})\) time, and the number of iterations until convergence is \(s_{\rm{t}}\),
the algorithm requires \(\mathcal{O}(|\mathcal{L}|ngs_{\rm{t}} + s_{\rm{s}})\) time.
This leads to a long processing time as the numbers of class labels \(|\mathcal{L}|\), data samples in each class \(g\), or unlabeled data samples \(n\) increases.

\section{Proposed Approach}
We first outlines our ideas to efficiently solve Problem~(\ref{eq:smoothed_dual}).
Next, we explain these ideas in detail and introduce our algorithm.
The proofs can be found in Appendix.

\subsection{Ideas}
The algorithm takes a long time to solve Problem~(\ref{eq:smoothed_dual}) when the numbers of class labels and data samples are large.
This is because it requires \(\mathcal{O}(|\mathcal{L}|ng)\) time to compute the gradient vectors \(\nabla \psi(\cdot)\) for each iteration until convergence, as described in the previous section.

To accelerate the gradient computation, we introduce two ideas.
The first idea is to skip the gradient computations for groups whose gradient vectors must be zero vectors.
As shown in Equation~(\ref{eq:gradient}), many of the gradient vectors during optimization are expected to be zero vectors owing to the soft-thresholding function.
On the basis of this observation, we approximately compute the soft-thresholding function at \(\mathcal{O}(|\mathcal{L}|(n+g))\) time and quickly check whether the gradient vectors must be zero vectors or not in advance of the exact gradient computation.

In the first idea, if a gradient vector turns out to be a nonzero vector, the checking procedure incurs unnecessary cost as it does not skip the computation of that gradient.
The second idea is to identify a subset of nonzero gradient vectors and compute them in the specified set without the checking procedure of the first idea.
As a result, it reduces the overhead of the first idea.

\subsection{Skipping Gradient Computations}
This section details the first idea of skipping the gradient computations for groups whose gradient vectors must be zero vectors.
To identify such groups, we introduce a criterion that is computed as follows:
\begin{definition}
\label{def_upper}
Suppose that \(z_{l,j} := |\!|[(\bm{\alpha} + \beta_{j} \bm{1}_{m} - \bm{c}_{j})_{[l]}]_{+}|\!|_{2}\).
Let \(\tilde{z}_{l,j}\), \(\tilde{\bm{\alpha}}\) and \(\tilde{\bm{\beta}}\) be old versions (snapshots) of \(z_{l,j}\), \(\bm{\alpha}\) and \(\bm{\beta}\) at some iteration in a gradient-based algorithm for optimization of Problem~(\ref{eq:smoothed_dual}).
We define \(\overline{Z}\in\mathbb{R}_+^{|\mathcal{L}|\times n}\) and \(\overline{z}_{l,j}\in\overline{Z}\) is computed as follows:
\begin{eqnarray}
\label{upper}
\textstyle{\overline{z}_{l,j}=\tilde{z}_{l,j}+|\!|[\Delta \bm{\alpha}_{[l]}]_{+} |\!|_{2}+\sqrt{g_l}[\Delta \beta_{j}]_{+}},
\end{eqnarray}
where \(\Delta \bm{\alpha}=\bm{\alpha}-\tilde{\bm{\alpha}}\), \(\Delta \bm{\beta}=\bm{\beta}-\tilde{\bm{\beta}}\), and
\(g_l\) represents the size of the \(l\)-th group.
\end{definition}
The snapshots are taken at regular intervals in the gradient-based algorithm,
e.g., every ten iterations, in the iterative method.
Note that \(z_{l,j}\) is a quantity that is used to compute the gradient vector \(\nabla\psi(\bm{f})_{[l]}\) in Equation~(\ref{eq:gradient})
because \(\nabla\psi(\bm{f})_{[l]} = [1 - \mu/|\!|\bm{f}_{[l]}^{+}|\!|_{2}]_{+} \bm{f}_{[l]}^{+} = [1 - \mu\gamma/z_{l,j}]_{+} \bm{f}_{[l]}^{+}\).
The following lemma shows that \(\overline{z}_{l,j}\) is an upper bound of \(z_{l,j}\):
\begin{lemma}[Upper Bound]
\label{lemma_upper}
\(\overline{z}_{l,j} \geq z_{l,j}\) holds when \(\overline{z}_{l,j}\) is computed using Equation~(\ref{upper}).
\end{lemma}
From the above lemma, we have the following lemma:
\begin{lemma}
\label{lemma_zero_vectors}
When \(\mu\gamma \geq \overline{z}_{l,j}\) holds,
we have \(\nabla\psi(\bm{\alpha} + \beta_{j} \bm{1}_{m} - \bm{c}_{j})_{[l]}=\bm{0}\).
\end{lemma}
The above lemma shows that we can identify groups whose gradient vectors must be zero vectors by utilizing the upper bound \(\overline{z}_{l,j}\).
The cost of computing \(\overline{z}_{l,j}\) is as follows:
\begin{lemma}
\label{lemma_cost_upper}
Given snapshots \(\tilde{Z}\), \(\tilde{\bm{\alpha}}\), and \(\tilde{\bm{\beta}}\), the cost of computing Equation~(\ref{upper}) for all elements in \(\overline{Z}\in\mathbb{R}_+^{|\mathcal{L}|\times n}\) is \(\mathcal{O}(|\mathcal{L}|(n+g))\) time.
\end{lemma}
The above lemma suggests that leveraging the upper bound allows us to efficiently identify groups whose gradient vectors must be zero vectors.
This is because the computation cost is \(\mathcal{O}(|\mathcal{L}|(n+g))\) time, while the original method requires \(\mathcal{O}(|\mathcal{L}|ng)\) time.

As described above, our upper bounds efficiently skip the gradient computations whose gradient vectors turn out to be zero vectors.
However, when a gradient vector turns out to be a nonzero vector, we must compute the upper bound as well as the exact gradient vector.
The next section shows a way to avoid this problem.

\subsection{Reduction of Overhead}
This section details our second idea: constructing a subset of groups whose gradient vectors turn out to be nonzero vectors.
In the first idea described above,
we compute the upper bound \(\overline{z}_{l,j}\) by using Equation~(\ref{upper}) and skip the gradient computation if \(\mu\gamma \geq \overline{z}_{l,j}\) holds from Lemma~\ref{lemma_upper}.
However, if \(\mu\gamma \geq \overline{z}_{l,j}\) does not hold, we must compute the gradient vector by using Equation~(\ref{eq:gradient}).
In this case, we have to compute both Equations~(\ref{upper}) and (\ref{eq:gradient}).
This may incur a large overhead if many gradient computations cannot be skipped.

To reduce the overhead of computing the upper bound, our second idea identifies a subset of nonzero gradient vectors during optimization and computes gradient vectors in the specified subset without computing the upper bound.
We introduce the following criterion to identify the subset by utilizing the same variables definitions in Definition~\ref{def_upper}:
\begin{definition}
\label{def_lower}
Suppose that \(\tilde{k}_{l,j} := |\!|(\tilde{\bm{\alpha}} + \tilde{\beta}_{j} \bm{1}_{m} - \bm{c}_{j})_{[l]}|\!|_{2}\)
and \(\tilde{o}_{l,j} := |\!|[(\tilde{\bm{\alpha}} + \tilde{\beta}_{j} \bm{1}_{m} - \bm{c}_{j})_{[l]}]_{-}|\!|_{2}\) are snapshots similar to those in Definition~\ref{def_upper}.
We define  \(\underline{Z}\in\mathbb{R}^{|\mathcal{L}|\times n}\) and \(\underline{z}_{l,j}\in\underline{Z}\) is computed as follows:
\begin{eqnarray}
\label{lower}
\textstyle{\underline{z}_{l,j}}&=&\textstyle{\tilde{k}_{l,j}-|\!|\Delta \bm{\alpha}_{[l]} |\!|_{2}-\sqrt{g_l}|\!|\Delta \beta_{j}|\!|_{2}} \nonumber \\
&&\textstyle{-\tilde{o}_{l,j}-|\!|[\Delta \bm{\alpha}_{[l]}]_{-} |\!|_{2}-\sqrt{g_l}|\!|[\Delta \beta_{j}]_{-}|\!|_{2}}.
\end{eqnarray}
\end{definition}
The following lemma shows that \(\underline{z}_{l,j}\) is a lower bound of \(z_{l,j}\) in Definition~\ref{def_upper}:
\begin{lemma}[Lower Bound]
\label{lemma_lower}
\(\underline{z}_{l,j} \leq z_{l,j}\) holds when \(\underline{z}_{l,j}\) is computed by Equation~(\ref{lower}).
\end{lemma}
From the above lemma, we have the following lemma:
\begin{lemma}
\label{lemma_nonzero_vectors}
When \(\mu\gamma < \underline{z}_{l,j}\) holds,
we obtain \(\nabla\psi(\bm{\alpha} + \beta_{j} \bm{1}_{m} - \bm{c}_{j})_{[l]}\neq\bm{0}\).
\end{lemma}
According to this lemma, we can identify nonzero gradient vectors by leveraging the lower bound \(\underline{z}_{l,j}\).
Therefore, we can construct the subset of groups that have nonzero gradient vectors as follows:
\begin{definition}
\label{def_set}
The set \(\mathbb{N}\) is constructed as
\begin{eqnarray}
\label{set}
\mathbb{N}=\{(l,j)\in\{1,...,|\mathcal{L}|\}\times\{1,...,n\}|\mu\gamma < \underline{z}_{l,j}\}.
\end{eqnarray}
\end{definition}
The computation cost is as follows:
\begin{lemma}
\label{lemma_cost_lower}
Given snapshots \(\tilde{K}\), \(\tilde{O}\), \(\tilde{\bm{\alpha}}\), and \(\tilde{\bm{\beta}}\), the computation cost of constructing the set \(\mathbb{N}\) is \(\mathcal{O}(|\mathcal{L}|(n+g))\) time.
\end{lemma}
The proof is similar to that of Lemma~\ref{lemma_cost_upper}.
After constructing the set \(\mathbb{N}\), the gradient vectors corresponding to \(\mathbb{N}\) are computed by using Equation~(\ref{eq:gradient}) without the checking procedure of Lemma~\ref{lemma_zero_vectors}.
As a result, our method can reduce the total overhead computing the upper bound since it does not compute the upper bound corresponding to \(\mathbb{N}\).
Note that although the cost of the lower bound is the same as that of the upper bound,
the total cost of the lower bound is smaller than that of the upper bound.
This is because the set \(\mathbb{N}\) is constructed at regular intervals during optimization.
The procedure is described in the next section.

\subsection{Algorithm}
\begin{algorithm}[t]
\begin{small}
\caption{Fast OT with Group Regularizer}
\label{alg_fastot}    
\begin{algorithmic}[1]
\State \(\bm{\alpha}\gets\bm{0}\), \(\bm{\beta}\gets\bm{0}\), \(\tilde{\bm{\alpha}}\gets\bm{0}\), \(\tilde{\bm{\beta}}\gets\bm{0}\), \(\mathbb{N}=\emptyset\)
\Repeat
\State apply a solver to Problem~\ref{eq:smoothed_dual} with the function \Call{GradPsi}{} in Algorithm~\ref{alg_grad} for \(r\) iterations;
\State compute \(|\!|\Delta \bm{\alpha}_{[l]} |\!|_{2}\) and \(|\!|[\Delta \bm{\alpha}_{[l]}]_{-} |\!|_{2}\) for \(l\in\mathcal{L}\);
\State compute \(\Delta \bm{\beta}\) and \([\Delta \bm{\beta}]_{-}\);
\State \(\mathbb{N}=\emptyset\);
\For{each \(j\in \{1,...n\}\)}
    \For{each \(l\in \mathcal{L}\)}
        \State compute the lower bound \(\underline{z}_{l,j}\) by Equation~(\ref{lower});
        \If{\(\underline{z}_{l,j}>\mu\gamma\)}
            \State add \((j,l)\) to \(\mathbb{N}\);
        \EndIf
    \EndFor
\EndFor
\State update the snapshots;
\Until{\(\bm{\alpha}\) and \(\bm{\beta}\) converge}
\end{algorithmic}
\end{small}
\end{algorithm}
\begin{algorithm}[t]
\begin{small}
\caption{Gradient Computation of \(\nabla\psi (\bm{\alpha} + \beta_{j} \bm{1}_{m} - \bm{c})\)}
\label{alg_grad}
\begin{algorithmic}[1]
\Function{GradPsi:}{}
\For{\((l,j)\in\mathbb{N}\)}
    \State compute \(\nabla \psi(\bm{\alpha}_{[l]} + \beta_{j} \bm{1}_{gl} - \bm{c}_{j[l]})\);
\EndFor
\State compute \(|\!|[\Delta \bm{\alpha}_{[l]}]_{+} |\!|_{2}\) for \(l\in\mathcal{L}\) and \([\Delta \bm{\beta}]_{+}\);
\For{\((l,j)\notin\mathbb{N}\)}
        \State compute the upper bound \(\overline{z}_{l,j}\) by Equation~(\ref{upper});
        \If{\(\overline{z}_{l,j}\leq\mu\gamma\)}
            \State \(\nabla \psi(\bm{\alpha}_{[l]} + \beta_{j} \bm{1}_{gl} - \bm{c}_{j[l]})\gets\bm{0}\);
        \Else
            \State compute \(\nabla \psi(\bm{\alpha}_{[l]} + \beta_{j} \bm{1}_{gl} - \bm{c}_{j[l]})\);
        \EndIf
\EndFor
\State \Return \([\nabla \psi(\bm{\alpha} + \beta_{1} \bm{1}_{m} - \bm{c}),...,\nabla \psi(\bm{\alpha} + \beta_{n} \bm{1}_{m} - \bm{c})]\);
\EndFunction
\end{algorithmic}
\end{small}
\end{algorithm}
Algorithm~\ref{alg_fastot} is the pseudocode of our algorithm.
Although it applies a solver, such as L-BFGS, to Problem~(\ref{eq:smoothed_dual}),
the solver efficiently computes \(\nabla\psi (\bm{\alpha} + \beta_{j} \bm{1}_{m} - \bm{c})\) by utilizing the upper bounds, as we described in Lemma~\ref{lemma_zero_vectors}.
Algorithm~\ref{alg_grad} is the gradient computation.
Here, the gradient vectors corresponding to \(\mathbb{N}\) are computed as usual by following Lemma~\ref{lemma_nonzero_vectors} (lines 2--4).
The other gradient vectors are computed with the upper bounds (lines 6--13).
Namely, we first compute the upper bound \(\overline{z}_{l,j}\) on the basis of Equation~(\ref{upper}) (line 7)
and skip the gradient computation if \(\overline{z}_{l,j}\leq\mu\gamma\) holds by following Lemma~\ref{lemma_zero_vectors} (lines 8--9).
If the inequality does not hold, the algorithm does not skip the computation (lines10--11).

In Algorithm~\ref{alg_fastot}, after it initializes the parameters, snapshots, and set \(\mathbb{N}\) (line 1),
it calls the solver with Algorithm~\ref{alg_grad} for \(r\) iterations (line 3).
We set \(r=10\) in this paper.
Next, it performs precomputations to obtain lower bounds for constructing the set \(\mathbb{N}\) whose gradient vectors are expected to be nonzero vectors (lines 4--5).
It uses Equation~(\ref{lower}) to compute the lower bound \(\underline{z}_{l,j}\) (line 9)
and adds the index to \(\mathbb{N}\) if \(\underline{z}_{l,j}>\mu\gamma\) holds on the basis of Lemma~\ref{lemma_nonzero_vectors} (lines 10--12).
Then, it updates the snapshots (line 15).
It repeats the above procedure until convergence (line 16).

The computation cost of Algorithm~\ref{alg_fastot} is as follows:
\begin{theorem}[Computation Cost]
\label{thm_cost}
In Algorithm~\ref{alg_fastot}, let \(s_{\rm{i}}\) and \(s_{\rm{n}}\) be the total number of \((j,l)\in\mathbb{N}\) and \((j,l)\notin\mathbb{N}\) for all iterations, respectively.
Suppose that \(s_{\rm{u}}\) is the total number of un-skipped gradient computations on line 11,
and \(s_{\rm{r}}\) is the total number of loops of lines 2--16.
If \(\mathcal{O}(s_{\rm{s}})\) time is the total computation cost other than the gradient computation of the solver in Algorithm~\ref{alg_fastot},
it requires \(\mathcal{O}((|\mathcal{L}|ns_{\rm{r}}+s_{\rm{i}}+s_{\rm{u}})g + s_{\rm{n}} + s_{\rm{s}})\) time.
\end{theorem}
If many gradient vectors become zero vectors during optimization, \(s_{\rm{i}}\) and \(s_{\rm{u}}\) are expected to be small.
Since the original method requires \(\mathcal{O}(|\mathcal{L}|ngrs_{\rm{r}} + s_{\rm{s}})\) time from \(s_{\rm{t}}=rs_{\rm{r}}\),
our method can be faster than the original method when the gradients are sparse.
Note that since our method reduces the cost of the gradient computation, it can be used with a wide range of solvers, such as L-BFGS.
Therefore, the total computation cost changes depending on the cost of the solver, i.e., \(s_{\rm{s}}\) in the total computation cost in Theorem~\ref{thm_cost}.

In terms of the optimization result, Algorithm~\ref{alg_fastot} has the following property:
\begin{theorem}[Optimization Result]
\label{thm_result}
Suppose that Algorithm~\ref{alg_fastot} has the same hyperparameters as those of the original method.
Then, Algorithm~\ref{alg_fastot} converges to the same solution and objective value as those of the original method.
\end{theorem}
The above theorem clearly holds because our method exactly computes all the nonzero gradient vectors in lines 2--4 and 6--13 in Algorithm~\ref{alg_grad}.
It suggests that our algorithm efficiently solves Problem~(\ref{eq:smoothed_dual}) without degrading accuracy.

\subsection{Convergence of Bounds}
Although Theorem~\ref{thm_result} guarantees the optimization results of the solution and the objective value after convergence,
our bounds also have advantageous properties for convergence.
Specifically, the upper bound has the following property:
\begin{theorem}[Convergence of Upper Bound]
\label{lemma_upper_error}
Let \(\overline{\epsilon}\) be an error bound defined as \(|\overline{z}_{l,j}-z_{l,j}|\).
Then, we have \(\overline{\epsilon}=0\) when \(\bm{\alpha}\) and \(\bm{\beta}\) reaches convergence through a gradient-based algorithm.
\end{theorem}
The above theorem suggests that the upper bound \(\overline{z}_{l,j}\) of \(z_{l,j}\) converges to the exact value of \(z_{l,j}\) if the algorithm converges.
This indicates that if \(\mu\gamma \geq z_{l,j}\) holds after convergence, \(\mu\gamma \geq \overline{z}_{l,j}\) always holds in Lemma~\ref{lemma_zero_vectors}.
When \(\mu\gamma \geq z_{l,j}\) holds,
we have \(\nabla\psi(\bm{\alpha} + \beta_{j} \bm{1}_{m} - \bm{c}_{j})_{[l]}=\bm{0}\) from the proof of Lemma~\ref{lemma_zero_vectors}.
Therefore, our upper bound can exactly identify all the gradient vectors that turn out to be zero vectors when the algorithm converges.

As for the lower bound, we have the following property:
\begin{theorem}[Convergence of Lower Bound]
\label{lemma_lower_error}
Let \(\underline{\epsilon}\) be an error bound defined as \(|z_{l,j}-\underline{z}_{l,j}|\).
Suppose that \(\bm{f}:=\bm{\alpha} + \beta_{j} \bm{1}_{m} - \bm{c}_{j}\).
If \(\bm{\alpha}\) and \(\bm{\beta}\) reaches convergence through a gradient-based algorithm, we have \(\underline{\epsilon}=| |\!|[\bm{f}_{[l]}]_{+}|\!|_{2} + |\!|[\bm{f}_{[l]}]_{-}|\!|_{2} - |\!|\bm{f}_{[l]}|\!|_{2} |\). 
\end{theorem}
Although the above theorem suggests that the error bound of the lower bound does not converge to zero,
this error bound has the following advantage:
\begin{corollary}[Convergence of Lower Bound]
\label{coro_lower_error}
If we have \([\bm{f}_{[l]}]_{+}=\bm{0}\) or \([\bm{f}_{[l]}]_{-}=\bm{0}\) for Theorem~\ref{lemma_lower_error},
we have \(\underline{\epsilon}=0\).
\end{corollary}
The above corollary indicates that the lower bound is tight when all the elements in \(\bm{f}_{[l]}\) are positive or negative.
This suggests that when a lot of gradient vectors turn out to be zero vectors during optimization, the lower bound is expected to be tight since \([\bm{f}_{[l]}]_{+}=\bm{0}\) will hold in many cases.
For nonzero gradient vectors, we can obtain a tight bound if all the elements in the gradient vector are nonzero because \([\bm{f}_{[l]}]_{-}=\bm{0}\) holds in Corollary~\ref{coro_lower_error}.

\section{Related Work}
To handle the group structure of the transportation plan, \citet{courty2014domain} utilized entropic and \(\ell_p-\ell_1\) regularization terms with \(p<1\).
Although the regularization term made the objective function nonconvex, they solved the optimization problem by using a majoration-minimization algorithm.
However, the method is only guaranteed to converge to local stationary points.
This drawback can be overcome by using a \(\ell_1-\ell_2\) regularization term instead of the  \(\ell_p-\ell_1\) regularization term \cite{courty2017optimal}.
This approach outperformed the previous method \cite{courty2014domain} on many benchmark tasks \cite{courty2017optimal} and its regularization term is now widely used to enable transportation plans to handle group structure \cite{redko2017theoretical,das2018sample,das2018unsupervised,li2020optimal,wang2021cross}.
However, it does not achieve group sparsity, as pointed out in the previous work \cite{blondel2018smooth}.
This is because the logarithm in the entropic regularization term keeps the values of the transportation plan in the strictly positive orthant.
Therefore, it is difficult to obtain a sparse transportation plan by using this approach.
On the other hand, \citet{blondel2018smooth} proposed another group-sparse regularizer, as we described in relation to Equation~(\ref{eq:reg_gl}).
Since this regularizer leverages the squared 2-norm and \(\ell_1-\ell_2\) regularization terms, it can avoid the above limitation.
Namely, their regularization term truly achieves group sparsity owing to the soft-thresholding function of Equation~(\ref{eq:gradient}).
In addition, the gradients can be computed in a closed-expression, and we can use various solvers on the optimization problem, as we explained in the preliminary section.
However, the cost of computing the gradient tends to be high when we handle large datasets  because it requires \(\mathcal{O}(|\mathcal{L}|ng)\) time to compute the gradient vectors for each iteration until convergence.

\section{Experiment}
We evaluated the processing time and accuracy to confirm the efficiency and effectiveness of our method.
\begin{figure}[t!]
\begin{center}
\includegraphics[viewport = 0.000000 0.000000 720.000000 360.000000, width=\columnwidth]{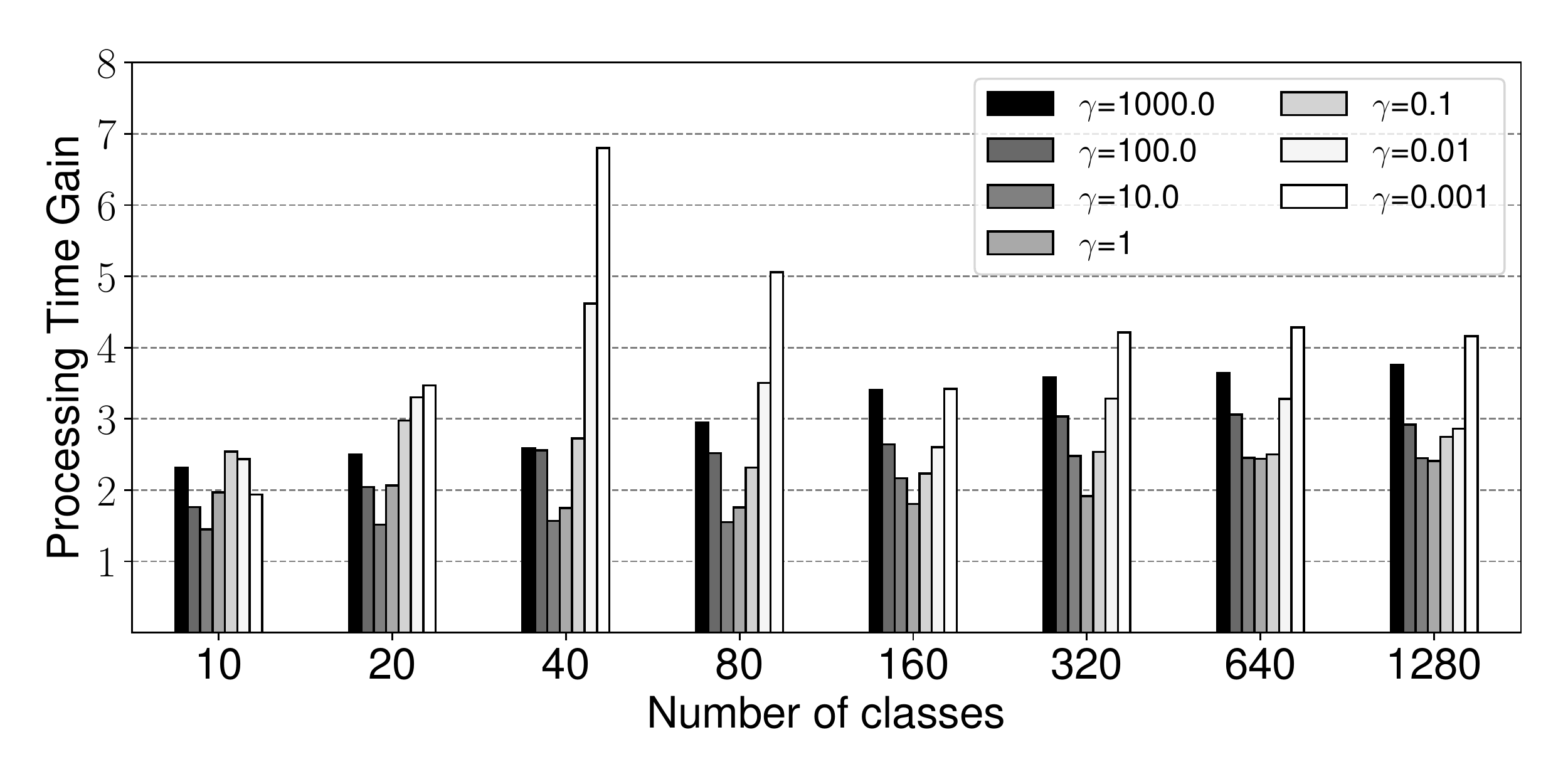} \\
 \caption{Processing time gain for each hyper parameter when the numbers of classes change.}
 \label{gammas_bar}
\end{center}
\end{figure}
\begin{figure}[t!]
\begin{center}
\includegraphics[viewport = 0.000000 0.000000 720.000000 360.000000, width=\columnwidth]{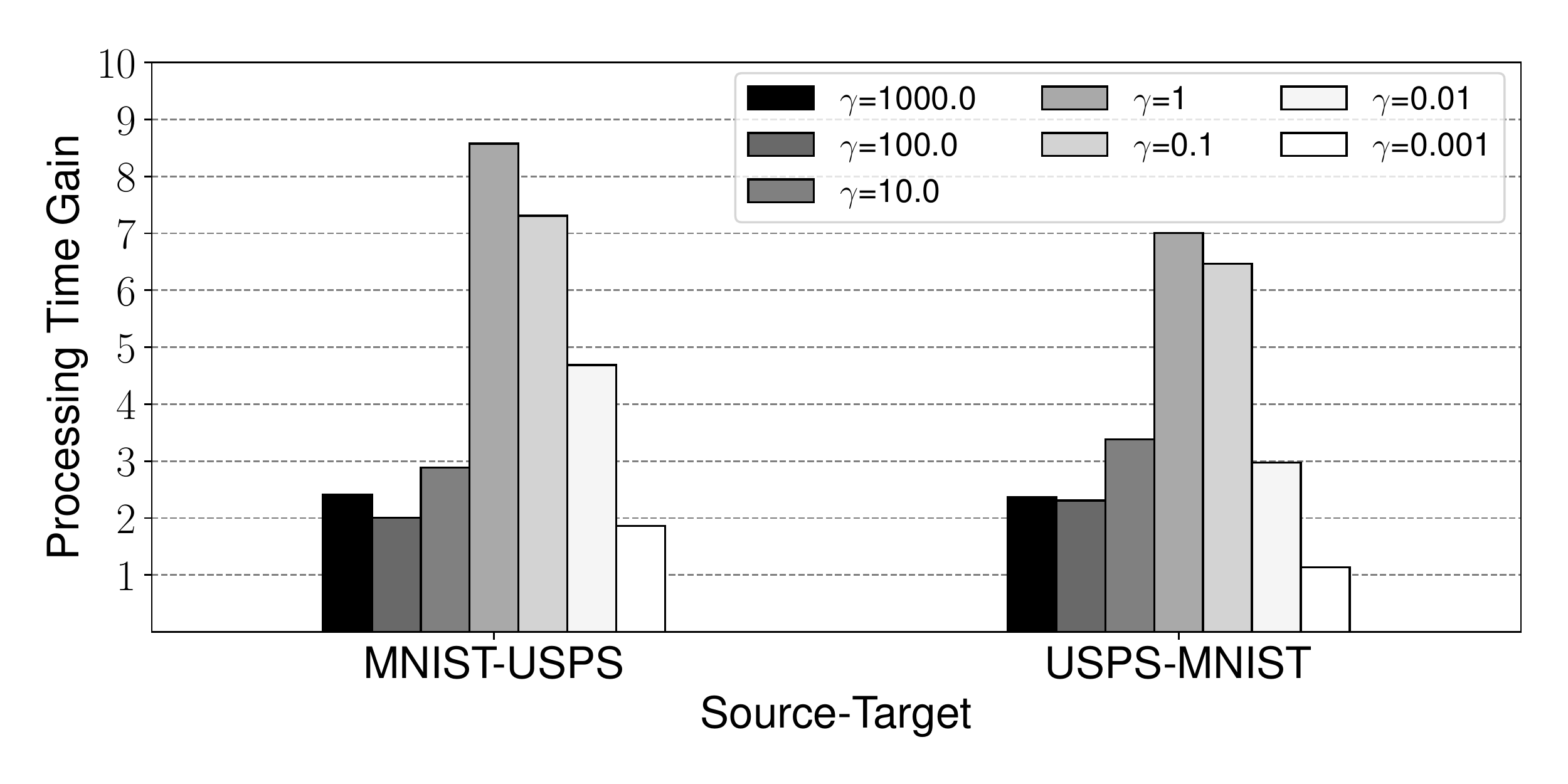} \\
 \caption{Processing time gain of 2 adaptation tasks on digit recognition.}
 \label{gammas_digit}
\end{center}
\end{figure}

\subsection{Datasets}
We created \(X^{S}\), \(X^{T}\), and \(\bm{y}^{S}\) from the following datasets including a synthetic dataset and visual adaptation datasets in accordance with the previous work \cite{courty2017optimal}:

\noindent{\bf Synthetic dataset with controlled number of class labels.} 
We used a simulated dataset with controlled numbers of class labels and data samples to show the efficiency.
We increased the number of class labels \(|\mathcal{L}|\) from 10 to 1,280 for \(X^{S}\).
The number of dimensions for each data sample was two.
Each class had ten data samples (\(g=10\)), which were generated from a standard normal distribution with a different mean for each class.
The means were computed as \((l\times 5.0, -5.0)\) for \(X^{S}\) and \((l\times 5.0, 5.0)\) for \(X^{T}\) where \(l\in \mathcal{L}\).
The labels \(l\) of \(X^{T}\) were only used to generate \(X^{T}\) and not used for the optimization.
Note that the numbers of data samples \(m\) and \(n\) automatically increased from 100 to 12,800 because we set \(n=m\) and \(m=|\mathcal{L}|g\).~\\
\noindent{\bf Digit recognition.} 
We used the digits datasets: USPS (U) \cite{hull1994database} and MNIST (M) \cite{lecun1998gradient} as \(X^{S}\) and \(X^{T}\).
Both datasets have ten class labels of digits.
We randomly sampled \(5,000\) images from each dataset.
The images in the datasets were resized to \(16\times16\).~\\    
\noindent{\bf Face recognition.}
We used the PIE dataset for the face recognition task \cite{gross2008multi}.
It contains \(32\times32\) images of 68 individuals taken under various conditions.
The number of classes is 68.
We used four domains in the dataset:
PIE05 (P5), PIE07 (P7), PIE09 (P9), and PIE29 (P29).
We created the combination of \(X^{S}\) and \(X^{T}\) by choosing two domains from these four domains.
As a result, we had 12 transportation problems for this dataset.
The numbers of images are 3332 (P5), 1629 (P7), 1632 (P9), and 1632 (P29).~\\ 
\noindent{\bf Object recognition.}
We used the Caltech-Office dataset for the object recognition task with ten class labels \cite{griffin2007caltech,gong12geodesic}.
The dataset consisted of four domains: {\em Caltech-256} (C), {\em Amazon} (A), {\em Webcam} (W) and {\em DSLR} (D).
Therefore, we had 12 transportation problems for this dataset.
The samples numbered 1123, 958, 295 and 157, respectively.
We used DeCAF\(_{6}\) \cite{donahue2014decaf} as the feature vectors.
They are activations of the fully connected layer of a convolutional neural network trained on ILSVRC-12.
The size of the vectors was 4096.
\begin{figure}[t!]
\begin{center}
\includegraphics[viewport = 0.000000 0.000000 720.000000 360.000000, width=\columnwidth]{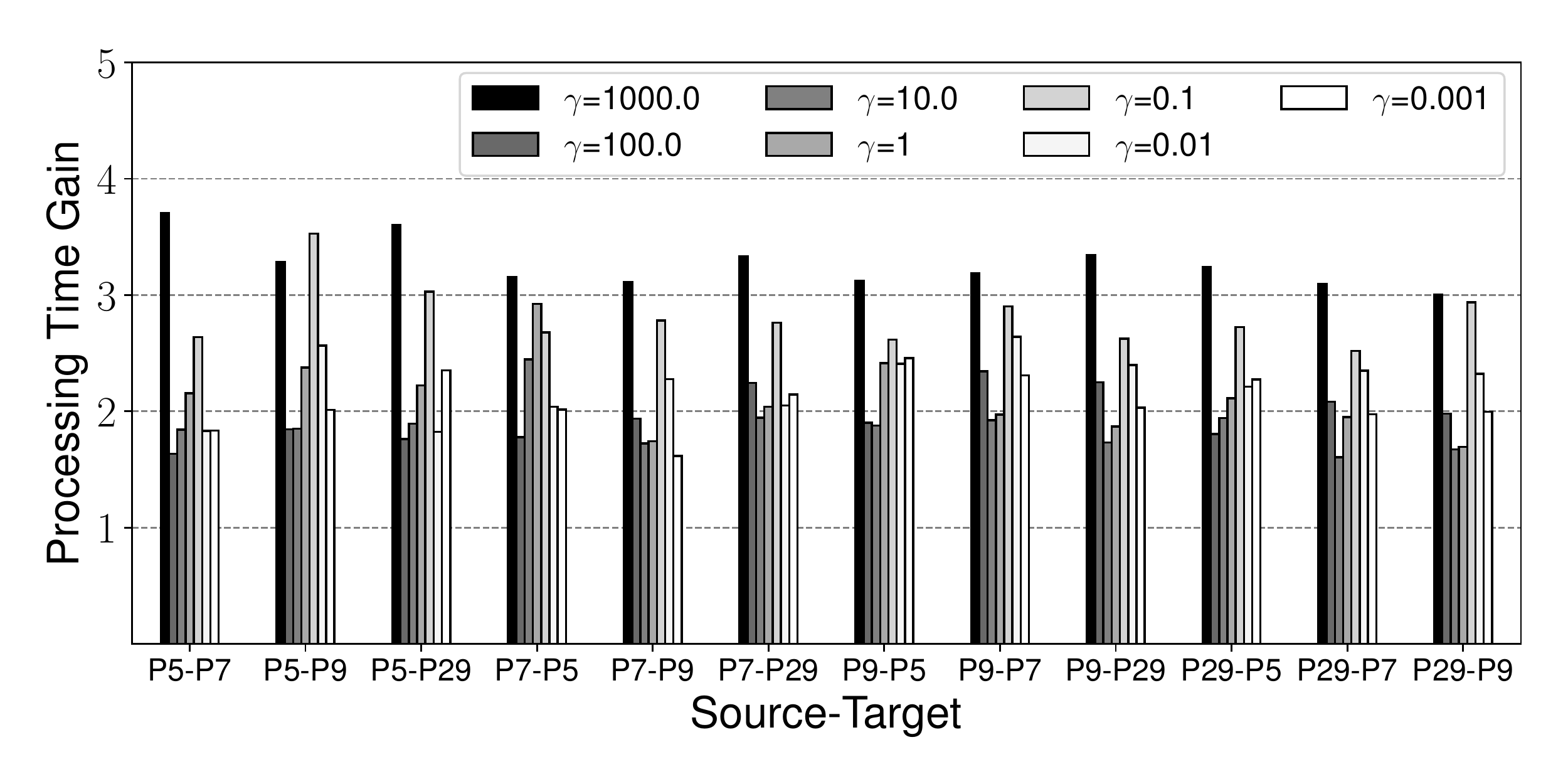} \\
 \caption{Processing time gain of 12 adaptation tasks on face recognition.}
 \label{gammas_pie}
\end{center}
\end{figure}
\begin{figure}[t!]
\begin{center}
\includegraphics[viewport = 0.000000 0.000000 720.000000 360.000000, width=\columnwidth]{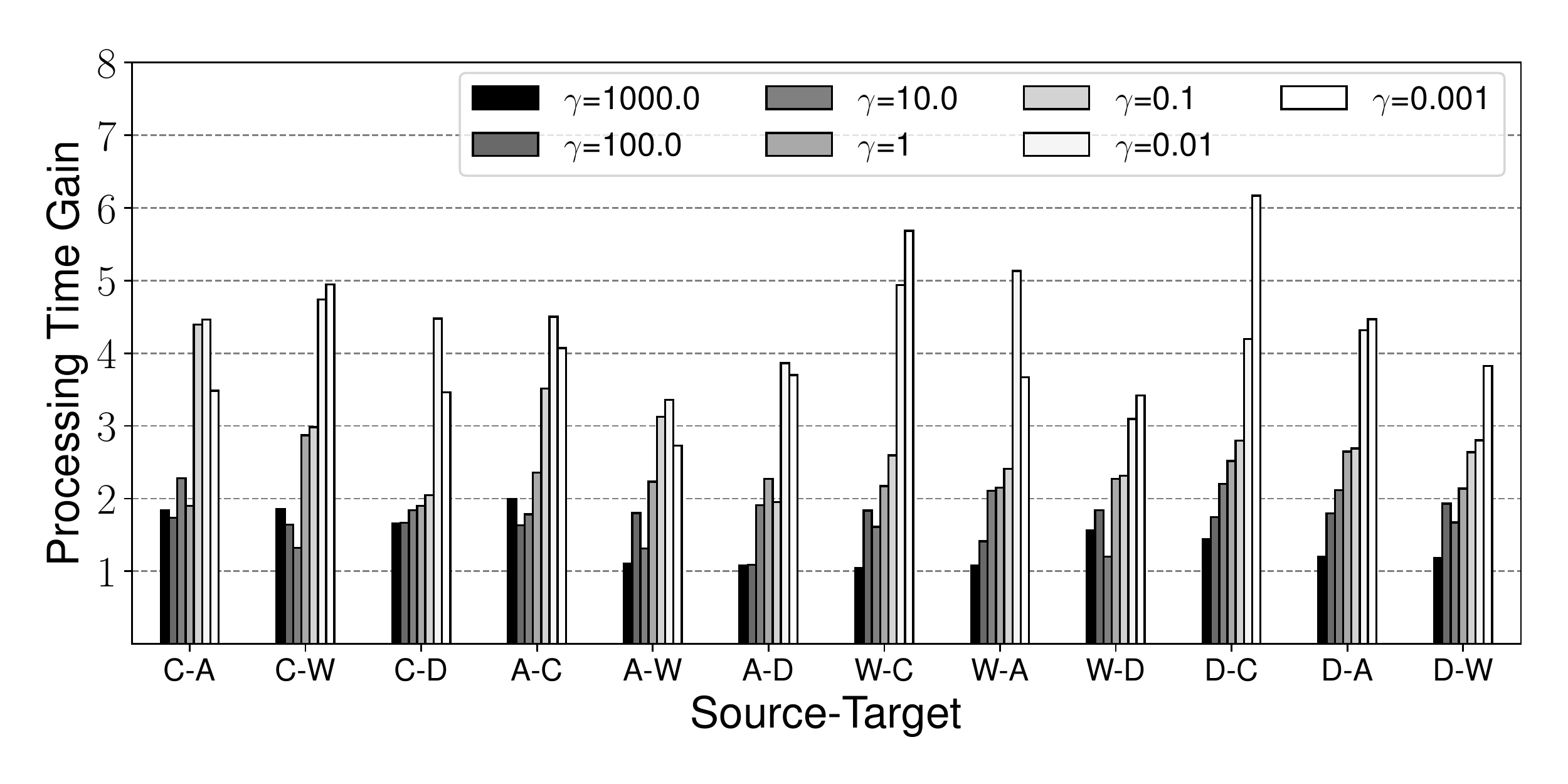} \\
 \caption{Processing time gain of 12 adaptation tasks on object recognition.}
 \label{gammas_caltech}
\end{center}
\end{figure}

\subsection{Experimental Setup}
We evaluated the processing time of solving Problem~(\ref{eq:smoothed_dual}) between different domains.
By following the implementation of the previous work \cite{blondel2018smooth}, we used L-BFGS and the hyperparameter \(\rho\in [0,1]\) instead of \(\mu\) in Equation~(\ref{eq:reg_gl}) to balance the regularization terms.
Namely, we utilized \(\Psi(\bm{t}_{j})=\gamma(\frac{1}{2}(1-\rho) |\!|\bm{t}_{j}|\!|_{2}^{2} + \rho \sum_{l\in\mathcal{L}}|\!|\bm{t}_{j[l]}|\!|_{2})\) instead of Equation~(\ref{eq:reg_gl}).
We evaluated the processing time and accuracy on combinational settings of the hyperparameters \(\rho=\{0.2, 0.4, 0.6, 0.8\}\) and \(\gamma=\{10^{3}, 10^{2}, 10^{1}, 10^{0}, 10^{-1}, 10^{-2}, 10^{-3}\}\) by following the previous work \cite{courty2017optimal,blondel2018smooth}.
Finally, we evaluated the total processing time of \(\rho=\{0.2, 0.4, 0.6, 0.8\}\) for each \(\gamma\) because \(\gamma\) adjusts the strength of the overall regularization terms.
We compared our method (ours) with the original method (origin) \cite{blondel2018smooth}.
Although we also tested the another method \cite{courty2017optimal}, we excluded it from the comparison since results could not be obtained for most of the hyperparameters.
This was due to the numerical instability of the Sinkhorn algorithm, as pointed out in the previous work \cite{schmitzer2019stabilized}.
In addition, that method could not achieve group sparsity due to the choice of the regularization term \cite{blondel2018smooth}.
Each experiment was conducted with one CPU core and 264 GB of main memory on a 2.20 GHz Intel Xeon server running Linux.

\subsection{Processing Time}
Figure~\ref{gammas_bar} shows the processing time gain on the synthetic dataset.
Our method is up to 6.8 times faster than the original method.
The gain increases as the numbers of class labels and data samples increase.
This is because our method efficiently skips the gradient computations corresponding to the class labels.
On the other hand, the checking procedure for skipping the gradient computations may become dominant when the numbers of class labels and data samples are small.
In this case, our second idea of reducing the overhead works.
As a result, our method turns out to be about twice as fast as the original method even with ten class labels.
We confirmed that our method without the second idea was slightly slower than the existing method when the number of class labels was 10.
The result suggests that the second idea helps to reduce the overhead especially for small numbers of class labels and data samples.

Figures~\ref{gammas_digit}, \ref{gammas_pie}, and \ref{gammas_caltech} show the processing time gain on the datasets for the digit, face, and object recognition tasks.
In these cases, our method is up to 8.6, 3.7, 6.2 times faster than the original method in each adaptation task.
The results suggest that our method is efficient even on real-world datasets.

Intuitively, our method has a large gain when \(\gamma\) is large because such a setting induces a sparse transportation plan, and the inequality in Lemma~\ref{lemma_upper} easily holds.
However, these figures suggest that the trend of gain for each \(\gamma\) is quite different depending on the dataset.
This is because some cases converge in a few iterations depending on the value of \(\gamma\) and dataset.
In such cases, the gain decreases since the number of gradient computations inherently small.

\subsubsection{Number of Gradient Computations.}
The aim of our idea is to skip gradient computations.
Therefore, we compared the number of gradient computations for the original method and our method.
Figure~\ref{count} shows the results for each \(\rho\) on MNIST-USPS dataset with \(\gamma=0.1\).
Our method could reduce the number of gradient computations by up to \(4.22\%\).
When the value of \(\rho\) is large, the magnitude of group-sparse regularization terms also becomes large.
Since group-sparsity is aggressively induced in such a setting, our method actually skips many gradient computations for \(\rho=0.8\) in Figure~\ref{count}.
This figure suggests that our method efficiently skips unnecessary gradient computations.
\begin{figure}[t!]
\begin{center}
\includegraphics[viewport = 0.000000 0.000000 720.000000 360.000000, width=\columnwidth]{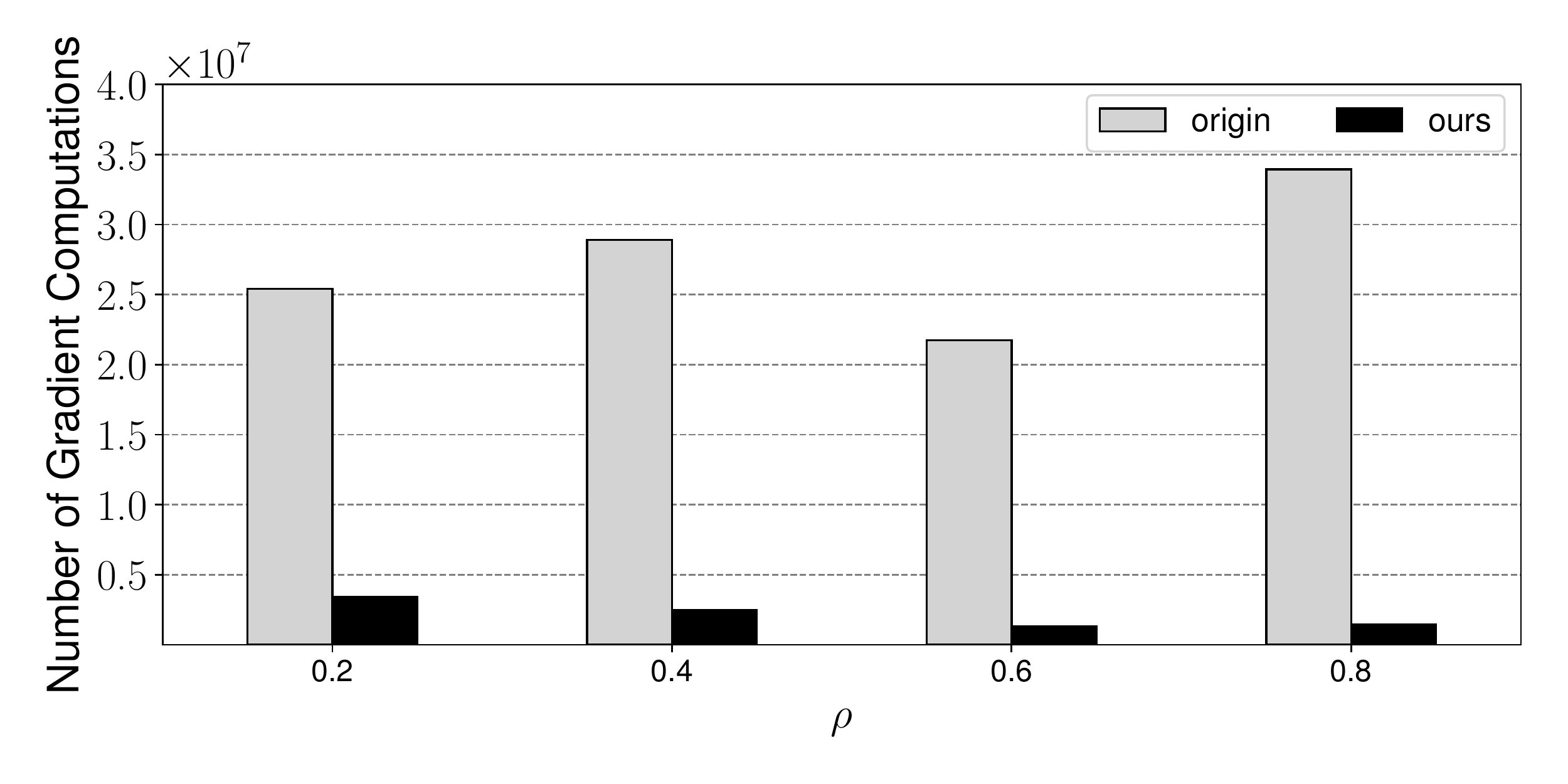} \\
 \caption{Numbers of gradient computations for each \(\rho\) on MNIST-USPS dataset with \(\gamma=0.1\).}
 \label{count}
\end{center}
\end{figure}

\begin{table}[t!]
\begin{center}
{
		\small
		\begin{tabular}{ccc}
		\toprule
			{Number of classes} & {Origin} & {Ours} \\
			\midrule
			10   & \(2.458\times10^{2}\) & \(2.458\times10^{2}\) \\
			20   & \(2.530\times10^{2}\) & \(2.530\times10^{2}\) \\
			40   & \(2.529\times10^{2}\) & \(2.529\times10^{2}\) \\
			80   & \(2.529\times10^{2}\) & \(2.529\times10^{2}\) \\
			160   & \(2.455\times10^{2}\) & \(2.455\times10^{2}\) \\
			320   & \(2.530\times10^{2}\) & \(2.530\times10^{2}\) \\
			640   & \(1.897\times10^{2}\) & \(1.897\times10^{2}\) \\
			1280   & \(2.529\times10^{2}\) & \(2.529\times10^{2}\) \\
			\bottomrule
		\end{tabular}		
}
\end{center}
\caption{Maximum objective values after convergence among all hyperparameters on the synthetic dataset.}
\label{tab:objective1}
\end{table}

\subsection{Accuracy} 
We also examined the values of the objective function of Problem~(\ref{eq:smoothed_dual}) after convergence to verify Theorem~\ref{thm_result}.
Here, we will mainly show results for the maximum objective function values after convergence among all hyperparameter combinations because Problem~(\ref{eq:smoothed_dual}) is a maximization problem.
The results on the synthetic dataset are listed in Table~\ref{tab:objective1}.
Our method achieves the same maximum objective values as those of the original method on all datasets for all hyperparameter combinations.
These experimental results verify our theoretical results and suggest that our method reduces the processing time without degrading accuracy.

\section{Conclusion}
We proposed fast regularized discrete optimal transport with group-sparse regularizers.
Our method exploits with two ideas.
The first idea is to safely skip the gradient computations whose gradient vectors must turn out to be zero vectors.
The second idea is to extract the gradient vectors that are expected to be nonzero.
Our method is guaranteed to return the same value of the objective function as that of the original method.
Experiments show that it is up to 8.6 times faster than the original method without degrading accuracy.

\bibliography{aaai23}

\begin{thebibliography}{28}
\providecommand{\natexlab}[1]{#1}

\bibitem[{Alaya et~al.(2019)Alaya, Berar, Gasso, and
  Rakotomamonjy}]{alaya2019screening}
Alaya, M.~Z.; Berar, M.; Gasso, G.; and Rakotomamonjy, A. 2019.
\newblock Screening {S}inkhorn {A}lgorithm for {R}egularized {O}ptimal
  {T}ransport.
\newblock In \emph{Advances in Neural Information Processing Systems
  (NeurIPS)}.

\bibitem[{Bertsekas(1999)}]{bertsekas1999nonlinear}
Bertsekas, D.~P. 1999.
\newblock \emph{Nonlinear {P}rogramming}.
\newblock Athena Scientific.

\bibitem[{Blondel, Seguy, and Rolet(2018)}]{blondel2018smooth}
Blondel, M.; Seguy, V.; and Rolet, A. 2018.
\newblock Smooth and {S}parse {O}ptimal {T}ransport.
\newblock In \emph{International {C}onference on {A}rtificial {I}ntelligence
  and {S}tatistics ({AISTATS})}.

\bibitem[{Courty, Flamary, and Tuia(2014)}]{courty2014domain}
Courty, N.; Flamary, R.; and Tuia, D. 2014.
\newblock Domain {A}daptation with {R}egularized {O}ptimal {T}ransport.
\newblock In \emph{Machine Learning and Knowledge Discovery in Databases -
  European Conference, {ECML} {PKDD}}.

\bibitem[{Courty et~al.(2017)Courty, Flamary, Tuia, and
  Rakotomamonjy}]{courty2017optimal}
Courty, N.; Flamary, R.; Tuia, D.; and Rakotomamonjy, A. 2017.
\newblock Optimal {T}ransport for {D}omain {A}daptation.
\newblock \emph{{IEEE} {T}ransactions on {P}attern {A}nalysis and {M}achine
  {I}ntelligence}, 39(9): 1853--1865.

\bibitem[{Cuturi(2013)}]{cuturi2013sinkhorn}
Cuturi, M. 2013.
\newblock Sinkhorn {D}istances: {L}ightspeed {C}omputation of {O}ptimal
  {T}ransport.
\newblock In \emph{Advances in {N}eural {I}nformation {P}rocessing {S}ystems
  ({NeurIPS})}.

\bibitem[{Das and Lee(2018{\natexlab{a}})}]{das2018sample}
Das, D.; and Lee, C. S.~G. 2018{\natexlab{a}}.
\newblock Sample-to-{S}ample {C}orrespondence for {U}nsupervised {D}omain
  {A}daptation.
\newblock \emph{Engineering Applications of Artificial Intelligence}, 73:
  80--91.

\bibitem[{Das and Lee(2018{\natexlab{b}})}]{das2018unsupervised}
Das, D.; and Lee, C. S.~G. 2018{\natexlab{b}}.
\newblock Unsupervised {D}omain {A}daptation {U}sing {R}egularized
  {H}yper-{G}raph {M}atching.
\newblock In \emph{{IEEE} International Conference on Image Processing
  ({ICIP})}.

\bibitem[{Donahue et~al.(2014)Donahue, Jia, Vinyals, Hoffman, Zhang, Tzeng, and
  Darrell}]{donahue2014decaf}
Donahue, J.; Jia, Y.; Vinyals, O.; Hoffman, J.; Zhang, N.; Tzeng, E.; and
  Darrell, T. 2014.
\newblock {D}e{CAF}: {A} {D}eep {C}onvolutional {A}ctivation {F}eature for
  {G}eneric {V}isual {R}ecognition.
\newblock In \emph{International {C}onference on {M}achine {L}earning
  ({ICML})}.

\bibitem[{Fujiwara et~al.(2016{\natexlab{a}})Fujiwara, Ida, Arai, Nishimura,
  and Iwamura}]{fujiwara2016fast}
Fujiwara, Y.; Ida, Y.; Arai, J.; Nishimura, M.; and Iwamura, S.
  2016{\natexlab{a}}.
\newblock Fast {A}lgorithm for the {L}asso based {L}1-{G}raph {C}onstruction.
\newblock \emph{Proc. {VLDB} Endow.}, 10(3): 229--240.

\bibitem[{Fujiwara et~al.(2016{\natexlab{b}})Fujiwara, Ida, Shiokawa, and
  Iwamura}]{fujiwara2016fast2}
Fujiwara, Y.; Ida, Y.; Shiokawa, H.; and Iwamura, S. 2016{\natexlab{b}}.
\newblock Fast {L}asso {A}lgorithm via {S}elective {C}oordinate {D}escent.
\newblock In \emph{Proceedings of the AAAI Conference on Artificial
  Intelligence}.

\bibitem[{Gangbo and Mc{C}ann(2000)}]{gangbo2006shape}
Gangbo, W.; and Mc{C}ann, R.~J. 2000.
\newblock Shape {R}ecognition via {W}asserstein {D}istance.
\newblock \emph{Quarterly of Applied Mathematics}, 58(4): 705--737.

\bibitem[{Gong et~al.(2012)Gong, Shi, Sha, and Grauman}]{gong12geodesic}
Gong, B.; Shi, Y.; Sha, F.; and Grauman, K. 2012.
\newblock Geodesic {F}low {K}ernel for {U}nsupervised {D}omain {A}daptation.
\newblock In \emph{{IEEE} {C}onference on {C}omputer {V}ision and {P}attern
  {R}ecognition ({CVPR})}.

\bibitem[{Griffin, Holub, and Perona(2007)}]{griffin2007caltech}
Griffin, G.; Holub, A.; and Perona, P. 2007.
\newblock {C}altech-256 {O}bject {C}ategory {D}ataset.
\newblock Technical report, {C}alifornia {I}nstitute of {T}echnology.

\bibitem[{Gross et~al.(2008)Gross, Matthews, Cohn, Kanade, and
  Baker}]{gross2008multi}
Gross, R.; Matthews, I.; Cohn, J.; Kanade, T.; and Baker, S. 2008.
\newblock Multi-{PIE}.
\newblock In \emph{IEEE International Conference on Automatic Face \& Gesture
  Recognition}.

\bibitem[{Hull(1994)}]{hull1994database}
Hull, J.~J. 1994.
\newblock A {D}atabase for {H}andwritten {T}ext {R}ecognition {R}esearch.
\newblock \emph{IEEE Transactions on Pattern Analysis and Machine
  Intelligence}, 16(5): 550--554.

\bibitem[{Ida, Fujiwara, and Kashima(2019)}]{ida2019fast}
Ida, Y.; Fujiwara, Y.; and Kashima, H. 2019.
\newblock Fast {S}parse {G}roup {L}asso.
\newblock In \emph{Advances in Neural Information Processing Systems
  ({NeurIPS})}.

\bibitem[{Ida et~al.(2020)Ida, Kanai, Fujiwara, Iwata, Takeuchi, and
  Kashima}]{ida2020fast}
Ida, Y.; Kanai, S.; Fujiwara, Y.; Iwata, T.; Takeuchi, K.; and Kashima, H.
  2020.
\newblock Fast {D}eterministic {CUR} {M}atrix {D}ecomposition with {A}ccuracy
  {A}ssurance.
\newblock In \emph{Proceedings of International Conference on Machine Learning
  ({ICML})}.

\bibitem[{Lecun et~al.(1998)Lecun, Bottou, Bengio, and
  Haffner}]{lecun1998gradient}
Lecun, Y.; Bottou, L.; Bengio, Y.; and Haffner, P. 1998.
\newblock Gradient-based {L}earning {A}pplied to {D}ocument {R}ecognition.
\newblock \emph{Proceedings of the IEEE}, 86(11): 2278--2324.

\bibitem[{Li et~al.(2020)Li, Ni, Zhu, Song, and Wu}]{li2020optimal}
Li, P.; Ni, Z.; Zhu, X.; Song, J.; and Wu, W. 2020.
\newblock Optimal {T}ransport with {D}imensionality {R}eduction for {D}omain
  {A}daptation.
\newblock \emph{Symmetry}, 12(12): 1994.

\bibitem[{Liu and Nocedal(1989)}]{liu1989on}
Liu, D.~C.; and Nocedal, J. 1989.
\newblock On the {L}imited {M}emory {BFGS} {M}ethod for {L}arge {S}cale
  {O}ptimization.
\newblock \emph{Mathematical Programming}, 45(1): 503--528.

\bibitem[{Lu et~al.(2021)Lu, Chen, Wang, and Qin}]{wang2021cross}
Lu, W.; Chen, Y.; Wang, J.; and Qin, X. 2021.
\newblock Cross-domain {A}ctivity {A}ecognition via {S}ubstructural {O}ptimal
  {T}ransport.
\newblock \emph{Neurocomputing}, 454: 65--75.

\bibitem[{Pitié, Kokaram, and Dahyot(2007)}]{pitie2007autimated}
Pitié, F.; Kokaram, A.~C.; and Dahyot, R. 2007.
\newblock Automated {C}olour {G}rading {U}sing {C}olour {D}istribution
  {T}ransfer.
\newblock \emph{{C}omputer {V}ision and {I}mage {U}nderstanding}, 107(1-2):
  123--137.

\bibitem[{Redko, Habrard, and Sebban(2017)}]{redko2017theoretical}
Redko, I.; Habrard, A.; and Sebban, M. 2017.
\newblock Theoretical {A}nalysis of {D}omain {A}daptation with {O}ptimal
  {T}ransport.
\newblock In \emph{Machine Learning and Knowledge Discovery in Databases -
  European Conference ({ECML} {PKDD})}.

\bibitem[{Russakovsky et~al.(2015)Russakovsky, Deng, Su, Krause, Satheesh, Ma,
  Huang, Karpathy, Khosla, Bernstein, Berg, and Fei-Fei}]{olga2015imagnet}
Russakovsky, O.; Deng, J.; Su, H.; Krause, J.; Satheesh, S.; Ma, S.; Huang, Z.;
  Karpathy, A.; Khosla, A.; Bernstein, M.; Berg, A.~C.; and Fei-Fei, L. 2015.
\newblock Image{N}et {L}arge {S}cale {V}isual {R}ecognition {C}hallenge.
\newblock \emph{International Journal of Computer Vision ({IJCV})}, 115(3):
  211--252.

\bibitem[{Schmitzer(2019)}]{schmitzer2019stabilized}
Schmitzer, B. 2019.
\newblock Stabilized {S}parse {S}caling {A}lgorithms for {E}ntropy
  {R}egularized {T}ransport {P}roblems.
\newblock \emph{{SIAM} Journal on Scientific Computing}, 41(3): A1443--A1481.

\bibitem[{Venturini, Baralis, and Garza(2017)}]{luca2017scaling}
Venturini, L.; Baralis, E.; and Garza, P. 2017.
\newblock Scaling {A}ssociative {C}lassification for {V}ery {L}arge {D}atasets.
\newblock \emph{Jornal of Big Data}, 4(44).

\bibitem[{Yuan and Lin(2006)}]{yuan2006model}
Yuan, M.; and Lin, Y. 2006.
\newblock Model {S}election and {E}stimation in {R}egression with {G}rouped
  {V}ariables.
\newblock \emph{Journal of the Royal Statistical Society}, 68(1): 49--67.

\end{thebibliography}

\newpage
\appendix
\setcounter{secnumdepth}{1}
\renewcommand{\thesection}{\Alph{section}}

 \makeatletter
    \renewcommand{\theequation}{%
    \thesection.\arabic{equation}}
    \@addtoreset{equation}{section}
\makeatother

\setcounter{table}{0}
\renewcommand{\thetable}{\Alph{table}}
\setcounter{figure}{0}
\renewcommand{\thefigure}{\Alph{figure}}
\setcounter{lemma}{0}
\renewcommand{\thelemma}{\Alph{lemma}}

\section{Proof of Lemma~\ref{lemma_upper}}
\begin{proof}
From the definition of \(z_{l,j}\), we obtain the following equation:
\begin{eqnarray}
z_{l,j} &=& |\!|[\bm{\alpha}_{[l]} + \beta_{j} \bm{1}_{g_l} - \bm{c}_{j[l]}]_{+}|\!|_{2} \nonumber \\
&=& |\!|[(\tilde{\bm{\alpha}} + \tilde{\beta}_{j} \bm{1}_{m} - \bm{c}_{j})_{[l]}+\Delta \bm{\alpha}_{[l]}+\Delta \beta_{j}\bm{1}_{g_l}]_{+}|\!|_{2}.  \nonumber
\end{eqnarray}
Here, \(0\leq\max(p+q, 0)\leq \max(p, 0)+\max(q, 0)\) holds.
From this inequality and the triangle inequality, we obtain the following upper bound if \(\bm{1}_{g_l}\) is a \(g_l\)-dimensional vector whose elements are ones:
\begin{eqnarray}
\!\!\!\!\!\!\!\!\!\!&&z_{l,j}\leq|\!|[(\tilde{\bm{\alpha}}\!+\!\tilde{\beta}_{j} \bm{1}_{m}\!-\!\bm{c}_{j})_{[l]}]_{+}\!+\![\Delta \bm{\alpha}_{[l]}]_{+}\!+\![\Delta \beta_{j}\bm{1}_{g_l}]_{+}|\!|_{2} \nonumber \\
\!\!\!\!\!\!\!\!\!\!&&\leq|\!|[(\tilde{\bm{\alpha}}\!+\!\tilde{\beta}_{j} \bm{1}_{m}\!-\!\bm{c}_{j})_{[l]}]_{+}|\!|_{2}\!+\!|\!|[\Delta \bm{\alpha}_{[l]}]_{+}|\!|_{2}\!+\!|\!|[\Delta \beta_{j}\bm{1}_{g_l}]_{+}|\!|_{2} \nonumber \\
\!\!\!\!\!\!\!\!\!\!&&=\tilde{z}_{l,j}+|\!|[\Delta \bm{\alpha}_{[l]}]_{+}|\!|_{2}+\sqrt{g_l}[\Delta \beta_{j}]_{+}= \overline{z}_{l,j}, \nonumber
\end{eqnarray}
which completes the proof.
\qed
\end{proof}

\section{Proof of Lemma~\ref{lemma_zero_vectors}}
Before we prove Lemma~\ref{lemma_zero_vectors}, we prove the following lemma:
\begin{lemma}
\label{lemma_condition}
When \(\mu\gamma \geq z_{l,j}\) holds,
we obtain \(\nabla\psi(\bm{\alpha} + \beta_{j} \bm{1}_{m} - \bm{c}_{j})_{[l]}=\bm{0}\).
\end{lemma}
\begin{proof}
Suppose that \(\bm{f}=\bm{\alpha} + \beta_{j} \bm{1}_{m} - \bm{c}_{j}\).
From Equation~(\ref{eq:gradient}),
\(\nabla\psi(\bm{f})_{[l]}=\bm{0}\) when
\(1 - \mu/|\!|\bm{f}_{[l]}^{+}|\!|_{2}\leq 0\) holds.
Since \(\textstyle{\bm{f}_{[l]}^{+} = \frac{1}{\gamma} [\bm{f}_{[l]}]_{+}}\) from the definition in Equation~(\ref{eq:gradient}) and \(|\!|[\bm{f}_{[l]}]_{+}|\!|_{2} = z_{l,j}\),
we obtain the desired inequality.
\qed
\end{proof}
We prove Lemma~\ref{lemma_zero_vectors} by utilizing the above lemma as follows:
\begin{proof}
When \(\mu\gamma \geq \overline{z}_{l,j}\) holds,
we have \(\mu\gamma \geq \overline{z}_{l,j}\geq z_{l,j}\) from Lemma~\ref{lemma_upper}.
Therefore, we obtain \(\nabla\psi(\bm{\alpha} + \beta_{j} \bm{1}_{m} - \bm{c}_{j})_{[l]}=\bm{0}\) from Lemma~\ref{lemma_condition}, as
\(\mu\gamma \geq z_{l,j}\) holds.
\qed
\end{proof}

\section{Proof of Lemma~\ref{lemma_cost_upper}}
\begin{proof}
Suppose that all the groups are of the same size \(g\), for simplicity.
The computations of \(|\!|[\Delta \bm{\alpha}_{[l]}]_{+} |\!|_{2}\) for all the groups and \([\Delta \bm{\beta}]_{+}\) require \(\mathcal{O}(|\mathcal{L}|g)\) and \(\mathcal{O}(n)\) times, respectively.
After the computations, Equation~(\ref{upper}) can be computed for all the elements in \(\overline{Z}\in\mathbb{R}_+^{|\mathcal{L}|\times n}\) in \(\mathcal{O}(|\mathcal{L}|n)\) time.
Therefore, the total computation cost of Equation~(\ref{upper}) is \(\mathcal{O}(|\mathcal{L}|(n+g))\) time given the snapshots.
\qed
\end{proof}

\section{Proof of Lemma~\ref{lemma_lower}}
\begin{proof}
From the definition of \(z_{l,j}\) and the triangle inequality,
we obtain the following equation:
\begin{eqnarray}
\label{lower1}
\!\!\!\!\!\!\!\!\!\!\!\!\!\!\!\!\!\!\!\!\!\!&&z_{l,j}\!=\!|\!|(\bm{\alpha}\!+\!\beta_{j} \bm{1}_{m}\!-\!\bm{c}_{j})_{[l]}\!-\![-(\bm{\alpha}\!+\!\beta_{j} \bm{1}_{m}\!-\!\bm{c}_{j})_{[l]}]_{+}|\!|_{2} \nonumber \\
\!\!\!\!\!\!\!\!\!\!\!\!\!\!\!\!\!\!\!\!\!\!&&\geq \!|\!|(\bm{\alpha}\!+\!\beta_{j} \bm{1}_{m}\!-\!\bm{c}_{j})_{[l]}|\!|_{2}\!-\!|\!|[-(\bm{\alpha}\!+\!\beta_{j} \bm{1}_{m}\!-\!\bm{c}_{j})_{[l]}]_{+}|\!|_{2}.
\end{eqnarray}
Here, we obtain the following inequality by utilizing a similar technique as in the proof of Lemma~\ref{lemma_upper}:
\begin{eqnarray}
\label{lower2}
|\!|(\bm{\alpha}\!+\!\beta_{j} \bm{1}_{m}\!-\!\bm{c}_{j})_{[l]}|\!|_{2}\geq\tilde{k}_{l,j}\!-\!|\!|\Delta \bm{\alpha}_{[l]} |\!|_{2}\!-\!\sqrt{g_l}|\!|\Delta \beta_{j}|\!|_{2}. 
\end{eqnarray}
In addition, we have the following inequality which is similar to the one in the proof of Lemma~\ref{lemma_upper}:
\begin{eqnarray}
\label{lower3}
&&\!\!\!\!\!\!\!\!\!\!\!\!\!\!|\!|[-(\bm{\alpha}\!+\!\beta_{j} \bm{1}_{m}\!-\!\bm{c}_{j})_{[l]}]_{+}|\!|_{2} \nonumber \\
&&\leq\tilde{o}_{l,j}\!+\!|\!|[\Delta \bm{\alpha}_{[l]}]_{-} |\!|_{2}\!+\!\sqrt{g_l}|\!|[\Delta \beta_{j}]_{-}|\!|_{2}.
\end{eqnarray}
Here, we have used \([-(\cdot)]_{+}=-[\cdot]_{-}\).
We obtain the inequality in the lemma by utilizing Equations~(\ref{lower1}), (\ref{lower2}) and (\ref{lower3}).
\qed
\end{proof}

\section{Proof of Lemma~\ref{lemma_nonzero_vectors}}
To prove Lemma~\ref{lemma_nonzero_vectors},
we introduce the following lemma:
\begin{lemma}
\label{lemma_lower_condition}
When \(\mu\gamma < z_{l,j}\) holds,
we obtain \(\nabla\psi(\bm{\alpha} + \beta_{j} \bm{1}_{m} - \bm{c}_{j})_{[l]}\neq \bm{0}\).
\end{lemma}
\begin{proof}
Suppose that \(\bm{f}=\bm{\alpha} + \beta_{j} \bm{1}_{m} - \bm{c}_{j}\).
From Equation~(\ref{eq:gradient}),
\(\nabla\psi(\bm{f})_{[l]}\neq \bm{0}\) when
\(1 - \mu/|\!|\bm{f}_{[l]}^{+}|\!|_{2}>0\) holds.
Since \(\textstyle{\bm{f}_{[l]}^{+} = \frac{1}{\gamma} [\bm{f}_{[l]}]_{+}}\) from the definition in Equation~(\ref{eq:gradient}) and \(|\!|[\bm{f}_{[l]}]_{+}|\!|_{2} = z_{l,j}\),
we obtain the desired inequality.
\qed
\end{proof}
We prove Lemma~\ref{lemma_nonzero_vectors} as follows:
\begin{proof}
When \(\mu\gamma < \underline{z}_{l,j}\) holds,
we have \(\mu\gamma < \underline{z}_{l,j}\leq z_{l,j}\) from Lemma~\ref{lemma_lower}.
Therefore, we obtain \(\nabla\psi(\bm{\alpha} + \beta_{j} \bm{1}_{m} - \bm{c}_{j})_{[l]}\neq\bm{0}\) from Lemma~\ref{lemma_lower_condition} , as \(\mu\gamma < z_{l,j}\) holds.
\qed
\end{proof}

\section{Proof of Theorem~\ref{thm_cost}}
\begin{proof}
From Lemma~\ref{lemma_cost_lower}, the extraction of the set \(\mathbb{N}\) requires \(\mathcal{O}(|\mathcal{L}|(n+g))\) time.
Since the extraction is performed \(s_{\rm{r}}\) times, the total extraction cost is \(\mathcal{O}(|\mathcal{L}|(n+g)s_{\rm{r}})\) time.
For the total gradient computation cost corresponding to \(\mathbb{N}\), \(\mathcal{O}(gs_{\rm{i}})\) time is required.
In addition, we need \(\mathcal{O}(s_{\rm{n}})\) time as the total computation cost of the upper bounds from the proof of Lemma~\ref{lemma_cost_upper}.
For the un-skipped gradient computations on line 11 in Algorithm~\ref{alg_grad}, we need \(\mathcal{O}(gs_{\rm{u}})\) time.
Furthermore, the updates of the snapshots require \(\mathcal{O}(|\mathcal{L}|ngs_{\rm{r}})\) time.
Since the solver requires \(\mathcal{O}(s_{\rm{s}}\)) other than the gradient computation, the computation cost of Algorithm~\ref{alg_fastot} is  \(\mathcal{O}((|\mathcal{L}|ns_{\rm{r}}+s_{\rm{i}}+s_{\rm{u}})g + s_{\rm{n}} + s_{\rm{s}})\) time.
\qed
\end{proof}

\section{Proof of Theorem~\ref{thm_result}}
\begin{proof}
From line 3 of Algorithm~\ref{alg_grad}, the gradient vectors corresponding to the set \(\mathbb{N}\) are exactly computed.
As for the other gradient vectors, when \(\overline{z}_{l,j}\leq\mu\gamma\) holds on line 8 in Algorithm~\ref{alg_grad}, their computations can be safely skipped, in accordance with Lemma~\ref{lemma_zero_vectors}.
Since the un-skipped gradient vectors are exactly computed on line 11, all the gradient vectors are exactly computed as in the original algorithm.
Therefore, Algorithm~\ref{alg_fastot} converges to the same solution and objective value as those of the original algorithm.
\qed
\end{proof}

\section{Proof of Theorem~\ref{lemma_upper_error}}
\begin{proof}
If \(\bm{\alpha}\) and \(\bm{\beta}\) reach convergence, \(\tilde{\bm{\alpha}}=\bm{\alpha}\) and \(\tilde{\bm{\beta}}=\bm{\beta}\) hold.
Then, \(\Delta \bm{\alpha}=\bm{0}\), \(\Delta \bm{\beta}=\bm{0}\), and \(\tilde{z}_{l,j}=z_{l,j}\) hold.
Therefore, we obtain \(\overline{z}_{l,j}=z_{l,j}\) from Equation~(\ref{upper}) and \(|\overline{z}_{l,j}-z_{l,j}|=0\), which completes the proof.
\qed
\end{proof}

\section{Proof of Theorem~\ref{lemma_lower_error}}
\begin{proof}
Similar to the proof of Theorem~\ref{lemma_upper_error},
we obtain \(\underline{z}_{l,j}=|\!|\bm{f}_{[l]}|\!|_{2}-|\!|[\bm{f}_{[l]}]_{-}|\!|_{2}\)
if \(\bm{\alpha}\) and \(\bm{\beta}\) reach convergence.
Since \(z_{l,j}=|\!|[\bm{f}_{[l]}]_{+}|\!|_{2}\), \(|\!|[\bm{f}_{[l]}]_{+}|\!|_{2}\geq |\!|\bm{f}_{[l]}|\!|_{2}-|\!|[\bm{f}_{[l]}]_{-}|\!|_{2}\) holds from Lemma~\ref{lemma_lower}.
Therefore, \(|z_{l,j}-\underline{z}_{l,j}|=| |\!|[\bm{f}_{[l]}]_{+}|\!|_{2} + |\!|[\bm{f}_{[l]}]_{-}|\!|_{2} - |\!|\bm{f}_{[l]}|\!|_{2} |\) holds, which completes the proof.
\qed
\end{proof}

\section{Proof of Corollary~\ref{coro_lower_error}}
\begin{proof}
From the triangle inequality, we obtain \(|\!|\bm{f}_{[l]}|\!|_{2} = |\!|[\bm{f}_{[l]}]_{+}+[\bm{f}_{[l]}]_{-}|\!|_{2} \leq|\!|[\bm{f}_{[l]}]_{+}|\!|_{2} + |\!|[\bm{f}_{[l]}]_{-}|\!|_{2}\).
From properties of the triangle inequality, \(|\!|\bm{f}_{[l]}|\!|_{2} = |\!|[\bm{f}_{[l]}]_{+}|\!|_{2} + |\!|[\bm{f}_{[l]}]_{-}|\!|_{2}\) holds for the case of \([\bm{f}_{[l]}]_{+}=\bm{0}\) or \([\bm{f}_{[l]}]_{-}=\bm{0}\).
In this case, \(|z_{l,j}-\underline{z}_{l,j}|=\bm{0}\) holds, which completes the proof.
\qed
\end{proof}

\section{Synthetic Dataset with Controlled Number of Samples per Class}
In the main paper, we evaluated the processing time on the simulated dataset with controlled numbers of class labels and data samples.
In this setting, \(n\), \(m\), and \(|\mathcal{L}|\) increased while the number of samples per class \(g\) was fixed at 10.
Therefore, we also evaluated the processing time gain when \(g\) increased from 10 to 160.
The number of class labels \(|\mathcal{L}|\) was fixed at 10, and we set \(n=m\) and \(m=|\mathcal{L}|g\).
Namely, the numbers of data samples \(n\) and \(m\) increased from 100 to 1,600.
The other settings were the same as the settings in the main paper.
Figure~\ref{gain_per_class} shows the result.
Our method is up to 6.5 times faster than the original method.
This is because our method requires \(\mathcal{O}(|\mathcal{L}|(n+g))\) time for the checking procedure, while the original method requires \(\mathcal{O}(|\mathcal{L}|ng)\) time for computing gradients.
In other words, for \(n\) and \(g\), the cost of our method is represented as their sum, whereas that of the original method is represented as their product.
Since \(n\) and \(g\) increase in this experiment, our method can efficiently solve the problem.
\begin{figure}[t!]
\begin{center}
\includegraphics[viewport = 0.000000 0.000000 720.000000 360.000000, width=\columnwidth]{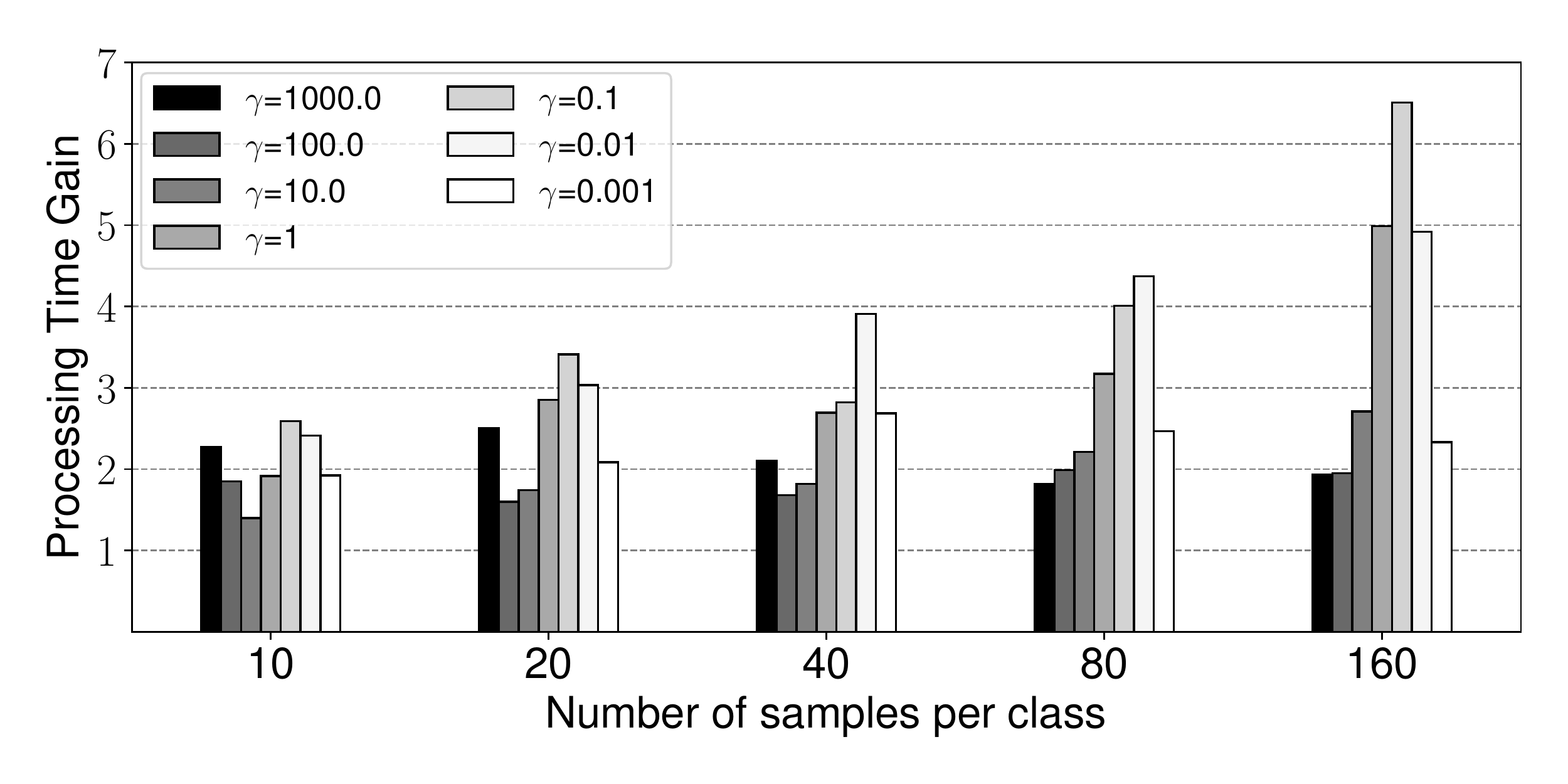} \\
 \caption{Processing time gain for each hyperparameter when the numbers of samples per class change.}
 \label{gain_per_class}
\end{center}
\end{figure}

\section{Convergence of Bounds}
\begin{figure}[t!]
\begin{center}
\includegraphics[viewport = 0.000000 0.000000 720.000000 360.000000, width=\columnwidth]{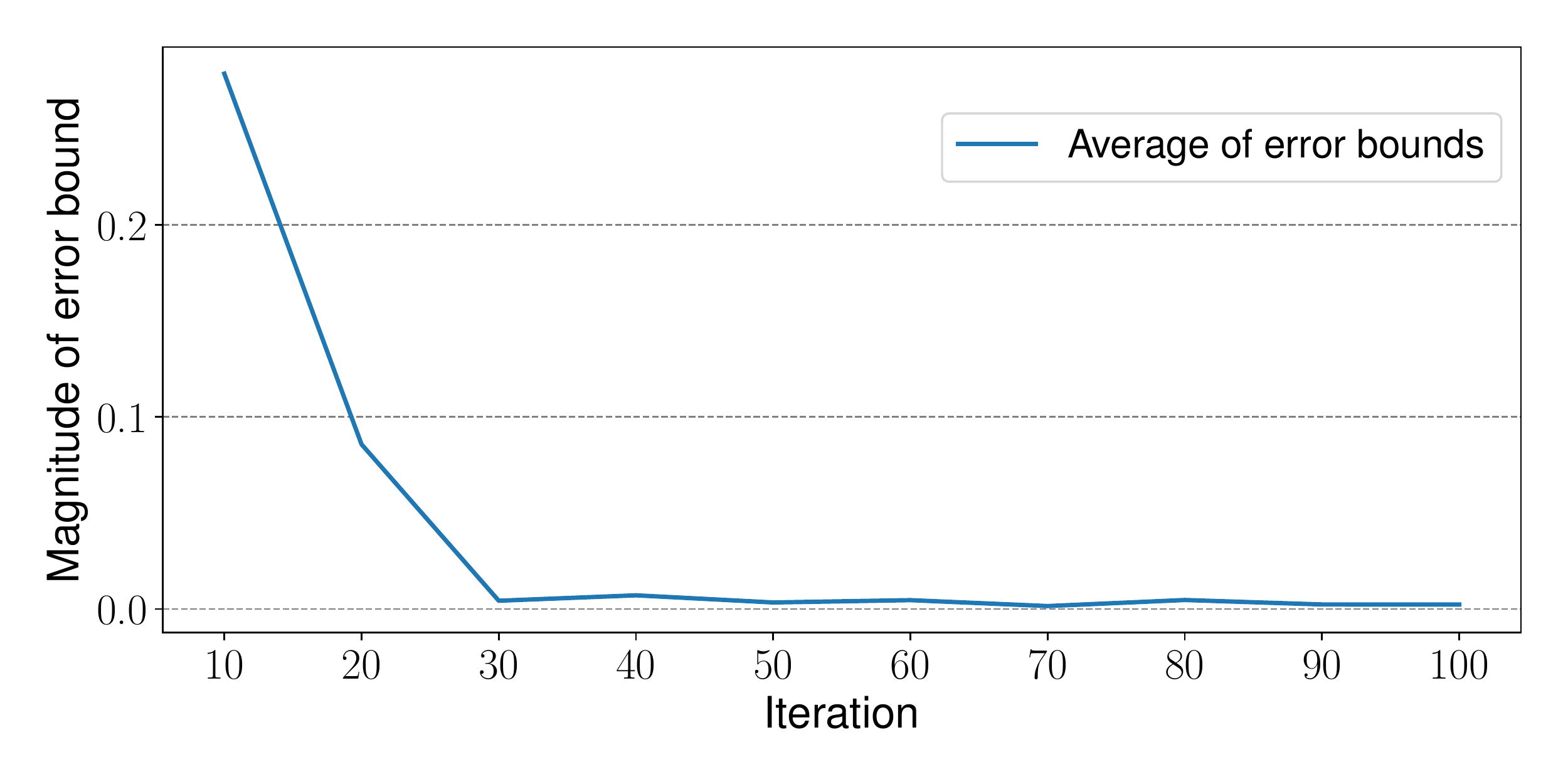} \\
 \caption{Error bounds on MNIST-USPS dataset with \(\gamma=0.1\) and \(\rho=0.8\).}
 \label{error}
\end{center}
\end{figure}
\begin{figure}[t!]
\begin{center}
\includegraphics[viewport = 0.000000 0.000000 720.000000 360.000000, width=\columnwidth]{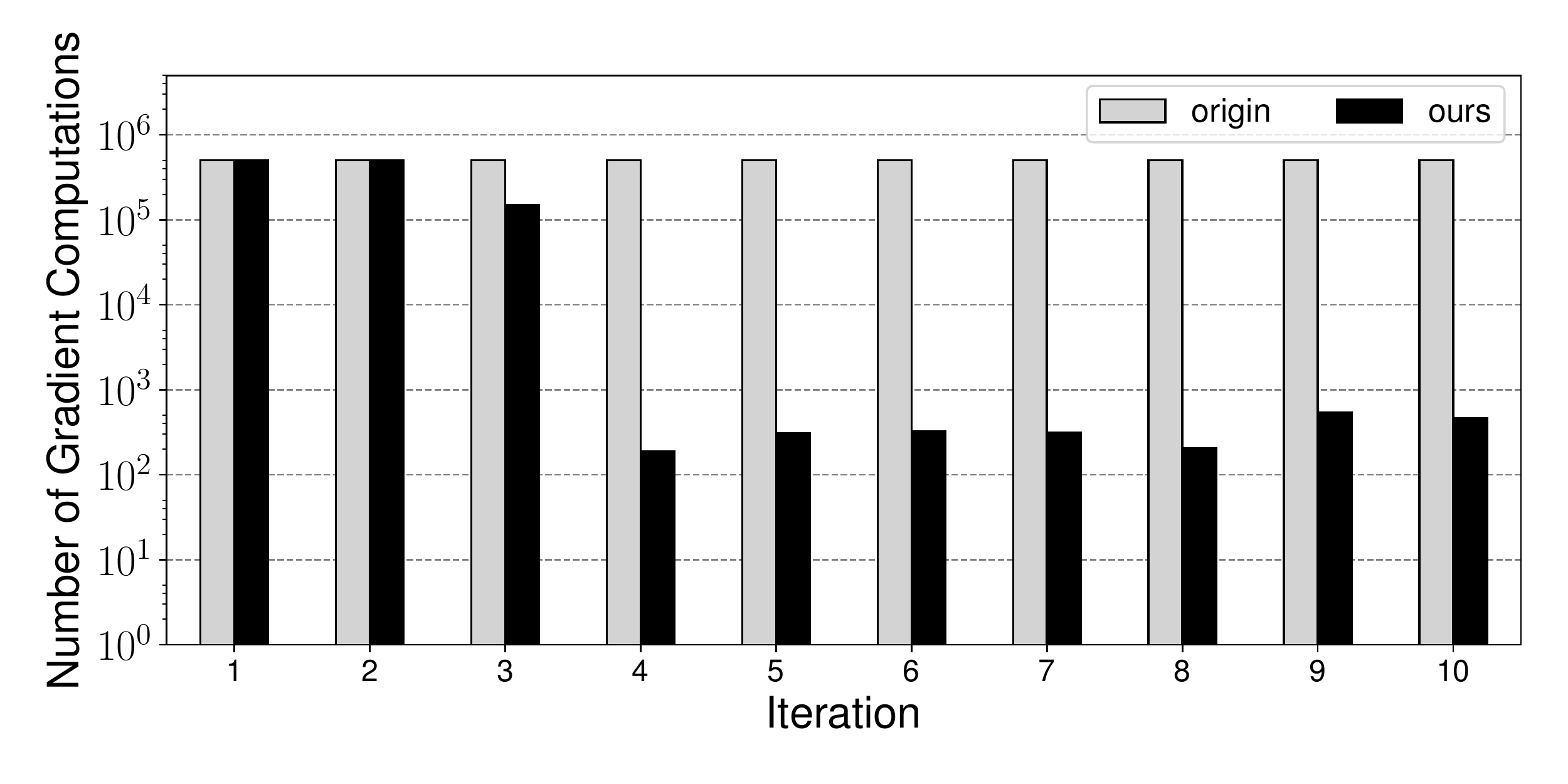} \\
 \caption{Numbers of gradient computations in a log scale for each iteration on MNIST-USPS dataset with \(\gamma=0.1\) and \(\rho=0.8\).}
 \label{count_iter}
\end{center}
\end{figure}
Although Theorem~\ref{lemma_upper_error} shows that the error bound of the upper bound converges to zero,
the error bound is expected to gradually approach zero as the optimization progresses.
This is because \(\Delta \bm{\alpha}\) and \(\Delta \bm{\beta}\) in the upper bound are also expected to approach zero during optimization.
Figure~\ref{error} shows magnitude of error bounds during optimization on MNIST-USPS dataset with \(\gamma=0.1\) and \(\rho=0.8\).
The result suggests that the error bound gradually approaches zero.
In other words, the upper bounds become gradually tight during optimization, and \(\mu\gamma\geq\overline{z}_{l,j}\) in Lemma~\ref{lemma_zero_vectors} will become easy to hold for zero gradients as the optimization progresses.

The above discussion also suggests that the efficiency of our method will increase during optimization.
This is because the inequality of \(\mu\gamma\geq\overline{z}_{l,j}\) will become easy to hold for zero gradients as the optimization progresses, and our method will effectively skip gradient computations.
To confirm this hypothesis, we evaluated  the number of gradient computations for each iteration.
Figure~\ref{count_iter} shows the numbers of the first ten iterations in a log scale on MNIST-USPS dataset with \(\gamma=0.1\) and \(\rho=0.8\).
Our method could reduce the number of computations by up to \(0.037\%\),
and skipped more computations as the number of iterations increases.
The result indicates that the upper bound effectively skips the gradient computations as the optimization progresses.

\section{Overhead Reduction with Lower Bound}
From Figure~\ref{gammas_bar}  in the main paper, the checking procedure with upper bounds may become dominant when the numbers of class labels and data samples are small.
Since the aim of our second idea is to reduce the overhead by utilizing lower bounds,
we compared our method with and without lower bounds on the simulated dataset.
The number of class labels was 10, and the other settings were the same as in the main paper.
Figure~\ref{wo_lower} shows the result.
Our method without the second idea is slightly slower than the original method for \(\gamma=0.001\) and \(0.01\).
On the other hand, our method with lower bounds turns out to be about twice as fast as the original method for \(\gamma=0.001\) and \(0.01\).
The result suggests the second idea helps to reduce the overhead of the first idea especially for small numbers of class labels and data samples.
\begin{figure}[t!]
\begin{center}
\includegraphics[viewport = 0.000000 0.000000 720.000000 360.000000, width=\columnwidth]{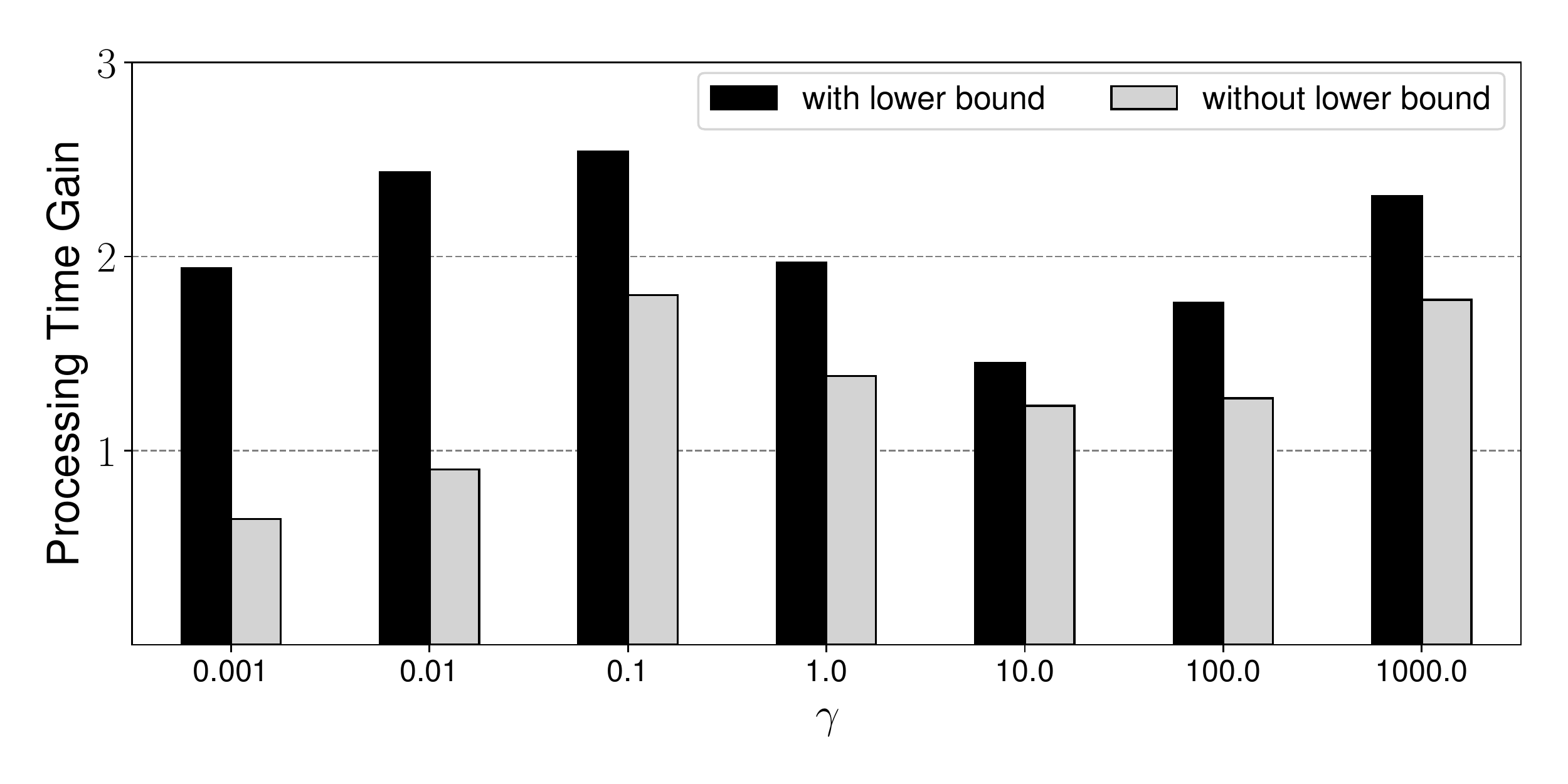} \\
 \caption{Processing time gain for our method with and without lower bounds. The number of class labels is 10.}
 \label{wo_lower}
\end{center}
\end{figure}
\end{document}